\def\year{2022}\relax
\documentclass[letterpaper]{article} % DO NOT CHANGE THIS
\usepackage{aaai22}  % DO NOT CHANGE THIS
\usepackage{times}  % DO NOT CHANGE THIS
\usepackage{helvet}  % DO NOT CHANGE THIS
\usepackage{courier}  % DO NOT CHANGE THIS
\usepackage[hyphens]{url}  % DO NOT CHANGE THIS
\usepackage{graphicx} % DO NOT CHANGE THIS
\usepackage{natbib}  % DO NOT CHANGE THIS AND DO NOT ADD ANY OPTIONS TO IT
\usepackage{caption} % DO NOT CHANGE THIS AND DO NOT ADD ANY OPTIONS TO IT
\DeclareCaptionStyle{ruled}{labelfont=normalfont,labelsep=colon,strut=off} % DO NOT CHANGE THIS
\usepackage{algorithm}
\usepackage{algorithmic}
\usepackage{microtype}
\usepackage{graphicx}
\usepackage{subcaption}
\usepackage{booktabs} % for professional tables
\usepackage{algorithm}
\usepackage{algorithmic}
\usepackage{color}
\usepackage{hyperref}

\usepackage{amssymb}
\usepackage{amsmath}
\usepackage{amsthm}
\usepackage{xspace}
\usepackage{dsfont}
\usepackage{bm}
\usepackage{stfloats}
\usepackage{tikz}
\tikzstyle{line}=[draw]
\usepackage{pgfplots,pgfplotstable}
\pgfplotsset{compat=1.12}
\usepackage[utf8]{inputenc}
\usepackage[export]{adjustbox}
\usepackage{threeparttable}

\usepackage{enumitem}
\newcommand{\compresslist}{%
    \setlength{\itemsep}{0pt}%
    \setlength{\topsep}{0pt}%
}

\DeclareRobustCommand{\eg}{e.g.,\@\xspace}
\DeclareRobustCommand{\ie}{i.e.,\@\xspace}
\DeclareRobustCommand{\wrt}{w.r.t.\@\xspace}

\newcommand{\de}{\,\mathrm{d}}
\newcommand{\sspace}{\mathcal{S}}
\newcommand{\aspace}{\mathcal{A}}
\newcommand{\reals}{\mathbb{R}}

\newcommand{\ind}{\mathds{1}}
\DeclareMathOperator*{\EV}{\mathbb{E}}
\DeclareMathOperator*{\Var}{\mathbb{V}ar}
\DeclareMathOperator{\var}{VaR_{\alpha}}
\DeclareMathOperator{\evar}{\widehat{VaR}_{\alpha}}
\DeclareMathOperator{\cvar}{CVaR_{\alpha}}
\newcommand{\cmp}{\mathcal{M}}
\newcommand{\vcmp}{\bm{\mathcal{M}}}
\newcommand{\diam}{\mathcal{D}}
\newcommand{\wass}{d_{W_1}}
\newcommand{\norm}[2]{\|{#1}\|_{#2}}
\newcommand{\os}{\overline{s}}
\newcommand{\oa}{\overline{a}}
\newcommand{\vtheta}{\bm{\theta}}
\newcommand{\Tau}{\mathcal{T}}

\usepackage{thmtools}
\usepackage{thm-restate}

\title{Unsupervised Reinforcement Learning in Multiple Environments}
\author {
        Mirco Mutti\textsuperscript{\rm 1,2,}\thanks{These authors contributed equally. ~~~The implementation of $\alpha$MEPOL is available at https://github.com/muttimirco/alphamepol.},
        Mattia Mancassola\textsuperscript{\rm 1,}$^*$,
        \textnormal{and} Marcello Restelli\textsuperscript{\rm 1} \\
}
\affiliations {
    \textsuperscript{\rm 1} Politecnico di Milano, Milan, Italy \\
    \textsuperscript{\rm 2} Universit\`a di Bologna, Bologna, Italy \\
    mirco.mutti@polimi.it, mattia.mancassola@mail.polimi.it, marcello.restelli@polimi.it \\
}

\begin{document}

\maketitle

\begin{abstract}
Several recent works have been dedicated to unsupervised reinforcement learning in a single environment, in which a policy is first pre-trained with unsupervised interactions, and then fine-tuned towards the optimal policy for several downstream supervised tasks defined over the same environment.
Along this line, we address the problem of unsupervised reinforcement learning in a class of multiple environments, in which the policy is pre-trained with interactions from the whole class, and then fine-tuned for several tasks in any environment of the class.
Notably, the problem is inherently multi-objective as we can trade off the pre-training objective between environments in many ways. In this work, we foster an exploration strategy that is sensitive to the most adverse cases within the class.
Hence, we cast the exploration problem as the maximization of the mean of a critical percentile of the state visitation entropy induced by the exploration strategy over the class of environments. Then, we present a policy gradient algorithm, $\alpha$MEPOL, to optimize the introduced objective through mediated interactions with the class. Finally, we empirically demonstrate the ability of the algorithm in learning to explore challenging classes of continuous environments and we show that reinforcement learning greatly benefits from the pre-trained exploration strategy \wrt learning from scratch.
\end{abstract}

\section{Introduction}
% RL limitation
The typical Reinforcement Learning~\citep[RL,][]{sutton2018reinforcement} setting involves a learning agent interacting with an environment in order to maximize a reward signal. In principle, the reward signal is a given and perfectly encodes the task. In practice, the reward is usually hand-crafted, and designing it to make the agent learn a desirable behavior is often a huge challenge. This poses a serious roadblock on the way of autonomous learning, as any task requires a costly and specific formulation, while the synergy between solving one RL problem and another is very limited.
% Unsupervised RL
To address this crucial limitation, several recent works~\cite{mutti2020policy, anonymous2021unsupervised, liu2021aps, seo2021state, yarats2021reinforcement} have been dedicated to \emph{unsupervised} RL. In this framework, originally envisioned in~\cite{hazan2019maxent, mutti2020intrinsically}, the agent first pre-trains its policy by taking a large amount of unsupervised interactions with the environment (\emph{unsupervised pre-training}). Then, the pre-trained policy is transferred to several downstream tasks, each of them defined through a reward function, and the agent has to learn an optimal policy by taking additional supervised interactions with the environment (\emph{supervised fine-tuning}). 
Whereas most of the existing works in unsupervised RL (\citet{campos2021finetuning} make for a notable exception) converged to a straightforward fine-tuning strategy, in which the pre-trained policy is employed as an exploratory initialization of a standard RL algorithm, there is lesser consensus on which unsupervised objective is best suited for the pre-training phase. 
Traditional intrinsic motivation bonuses that were originally designed to address exploration in supervised RL~\citep[\eg][]{pathak2017curiosity, burda2019rnd} can be employed in the unsupervised RL setting as well~\cite{laskin2021urlb}. However, these bonuses are designed to vanish over time, which makes it hard to converge to a stable policy during the unsupervised pre-training. 
The \emph{Maximum State Visitation Entropy}~\citep[MSVE,][]{hazan2019maxent} objective, which incentives the agent to learn a policy that maximizes the entropy of the induced state visitation, emerged as a powerful alternative in both continuous control and visual domains~\cite{laskin2021urlb}. The intuition underlying the MSVE objective is that a pre-trained exploration strategy should visit with high probability any state where the agent might be rewarded in a subsequent supervised task, so that the fine-tuning to the optimal policy is feasible. Although unsupervised pre-training methods effectively reduce the reliance on a reward function and lead to remarkable fine-tuning performances \wrt RL from scratch, all of the previous solutions to unsupervised RL assume the existence of a single environment.

% General problem
In this work, we aim to push the generality of this framework even further, by addressing the problem of \emph{unsupervised RL in multiple environments}. In this setting, during the pre-training the agent faces a class of reward-free environments that belong to the same domain but differ in their transition dynamics. At each turn of the learning process, the agent is drawn into an environment within the class, where it can interact for a finite number of steps before facing another turn. The ultimate goal of the agent is to pre-train an exploration strategy that helps to solve \emph{any} subsequent fine-tuning task that can be specified over \emph{any} environment of the class.

% Our contributions
Our contribution to the problem of unsupervised RL in multiple environments is three-fold: First, we frame the problem into a tractable \emph{formulation} (Section~\ref{sec:problem}), then, we propose a \emph{methodology} to address it (Section~\ref{sec:method}), for which we provide a thorough \emph{empirical} evaluation (Section~\ref{sec:experiments}).
Specifically, we extend the pre-training objective to the multiple-environments setting. Notably, when dealing with multiple environments the pre-training becomes a \emph{multi-objective} problem, as one could establish any combination of preferences over the environments. Previous unsupervised RL methods would blindly optimize the average of the pre-training objective across the class, implicitly establishing a uniform preference.
Instead, in this work we consider the mean of a critical percentile of the objective function, \ie its Conditional Value-at-Risk~\citep[CVaR,][]{rockafellar2000optimization} at level $\alpha$, to prioritize the performance in particularly rare or adverse environments. 
In line with the MSVE literature, we chose the CVaR of the induced state visitation entropy as the pre-training objective, and we propose a policy gradient algorithm~\citep{deisenroth2013survey}, \emph{$\alpha$-sensitive Maximum Entropy POLicy optimization} ($\alpha$MEPOL), to optimize it via mere interactions with the class of environments. As in recent works~\citep{mutti2020policy,anonymous2021unsupervised,seo2021state}, the algorithm employs non-parametric methods to deal with state entropy estimation in continuous and high-dimensional environments. Then, it leverages these estimated values to optimize the CVaR of the entropy by following its policy gradient~\citep{tamar2015optimizing}.
Finally, we provide an extensive experimental analysis of the proposed method in both the unsupervised pre-training over classes of multiple environments, and the supervised fine-tuning over several tasks defined over the class. The exploration policy pre-trained with $\alpha$MEPOL allows to solve sparse-rewards tasks that are impractical to learn from scratch, while consistently improving the performance of a pre-training that is blind to the unfavorable cases.

\section{Related Work}
\label{sec:related_work}

In this section, we revise the works that relates the most with the setting of unsupervised RL in multiple environments. A more comprehensive discussion can be found in Appendix~\ref{apx:related_work}.

In a previous work, \citet{rajendran2020should} considered a learning process composed of agnostic pre-training (called a \emph{practice}) and supervised fine-tuning (a \emph{match}) in a class of environments. However, in their setting the two phases are alternated, and the supervision signal of the matches allows to learn the reward for the practice through a meta-gradient.

\citet{parisi2021interesting} addresses the unsupervised RL in multiple environments concurrently to our work. Whereas their setting is akin to ours, they come up with an essentially orthogonal solution. Especially, they consider a pre-training objective inspired by count-based methods~\cite{bellemare2016unifying} in place of our entropy objective. Whereas they design a specific bonus for the multiple-environments setting, they essentially establish a uniform preference over the class  instead of prioritizing the worst-case environment as we do.

Finally, our framework resembles the \emph{meta-RL} setting \citep{finn2017model}, in which we would call \emph{meta-training} the unsupervised pre-training, and \emph{meta-testing} the supervised fine-tuning. However, none of the existing works combine unsupervised meta-training \citep{gupta2018unsupervised} with a multiple-environments setting.

\section{Preliminaries}
\label{sec:preliminaries}

A vector $\bm{v}$ is denoted in bold, and $v_i$ stands for its $i$-th entry.

\paragraph{Probability and Percentiles}
Let $X$ be a random variable distributed according to a cumulative density function (cdf) $F_X (x) = Pr (X \leq x)$. We denote with $\EV [X]$, $\Var [X]$ the expected value and the variance of $X$ respectively.
Let $\alpha \in (0, 1)$ be a confidence level, we call the $\alpha$-percentile (shortened to $\alpha$\%) of the variable $X$ its Value-at-Risk (VaR), which is defined as
$$
    \var (X) = \inf \big\lbrace x \;|\; F_X (x) \geq \alpha \big\rbrace.
$$
Analogously, we call the mean of this same $\alpha$-percentile the Conditional Value-at-Risk (CVaR) of $X$,
$$
    \cvar (X) = \EV \big[ X \;|\; X \leq \var (X) \big].
$$

\paragraph{Markov Decision Processes}
A Controlled Markov Process (CMP) is a tuple $\mathcal{M} := (\sspace, \aspace, P, D)$, where $\sspace$ is the state space, $\aspace$ is the action space, the transition model $P(s' | a, s)$ denotes the conditional probability of reaching state $s'$ when selecting action $a$ in state $s$, and $D$ is the initial state distribution.
The behavior of an agent is described by a policy $\pi (a | s)$, which defines the probability of taking action $a$ in $s$. Let $\Pi$ be the set of all the policies.
Executing a policy $\pi$ in a CMP over $T$ steps generates a trajectory $\tau = (s_{0,\tau}, a_{0,\tau}, \ldots, a_{T-2,\tau}, s_{T - 1, \tau})$ such that
$
    p_{\pi, \cmp} (\tau) = D(s_{0, \tau}) \prod^{T - 1}_{t = 0} \pi(a_{t,\tau} | s_{t,\tau}) P (s_{t+1,\tau}| s_{t,\tau}, a_{t,\tau})
$
denotes its probability.
We denote the state-visitation frequencies induced by $\tau$ with $ d_\tau (s) = \frac{1}{T} \sum_{t=0}^{T-1} \ind (s_{t,\tau} = s)$, and we call $d_\pi^{\cmp} = \EV_{\tau \sim p_{\pi, \cmp}} [d_{\tau}]$ the marginal state distribution. We define the differential entropy~\citep{shannon1948entropy} of $d_\tau$ as 
$ 
	H (d_\tau) = - \int_\sspace d_\tau (s) \log d_{\tau} (s) \de s.
$
For simplicity, we will write $H (d_\tau)$ as a random variable $H_\tau \sim \delta (h - H(d_\tau)) p_{\pi, \cmp} (\tau)$, where $\delta(h)$ is a Dirac delta.

By coupling a CMP $\cmp$ with a reward function $R$ we obtain a Markov Decision Process~\citep[MDP,][]{puterman2014markov} $\cmp^R := \cmp \cup R$. Let $R(s, a)$ be the expected immediate reward when taking $a \in \aspace$ in $s \in \sspace$ and let $R(\tau) = \sum_{t=0}^{T -1} R (s_{t, \tau})$, the \emph{performance} of a policy $\pi$ over the MDP $\cmp^R$ is defined as
\begin{equation}
    \mathcal{J}_{\cmp^R} (\pi) = \EV_{\tau \sim p_{\pi,\cmp}} \big[ R (\tau) \big].
    \label{eq:rl_objective}
\end{equation}
The goal of reinforcement learning~\cite{sutton2018reinforcement} is to find an optimal policy $\pi^*_{\mathcal{J}} \in \arg \max \mathcal{J}_{\cmp^R} (\pi)$ through sampled interactions with an unknown MDP $\cmp^R$.

\section{Unsupervised RL in Multiple Environments}
\label{sec:problem}
\begin{figure*}[t]
    \centering
    \includegraphics[width=0.96\textwidth]{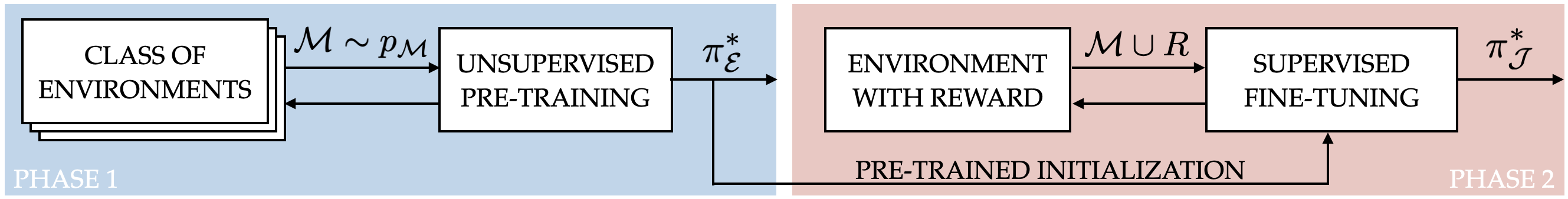}
    \caption{On the left, we highlight the unsupervised pre-training, in which the agent iteratively interacts with a CMP $\cmp \in \vcmp$ drawn from $p_\cmp$. The pre-trained policy $\pi^*_{\mathcal{E}}$ conveys the initialization to the subsequent supervised fine-tuning (on the right), which outputs a reward maximizing policy $\pi_{\mathcal{J}}^*$ for an MDP $\cmp \cup R$ that pairs $\cmp \in \vcmp$ with an arbitrary reward function $R$.}
    \label{fig:problem}
    %\vspace{-0.3cm}
\end{figure*}
Let $\vcmp = \{\cmp_1, \ldots, \cmp_{I}\}$ be a class of unknown CMPs, in which every  element $\cmp_i = (\sspace, \aspace, P_i, D)$ has a specific transition model $P_i$, while $\sspace, \aspace, D$ are homogeneous across the class. At each turn, the agent is able to interact with a single environment $\cmp \in \vcmp$. The selection of the environment to interact with is mediated by a distribution $p_{\vcmp}$ over $\vcmp$, outside the control of the agent. The aim of the agent is to pre-train an exploration strategy that is general across all the MDPs $\cmp^R$ one can build upon $\vcmp$. In a single-environment setting, this problem has been assimilated to learning a policy that maximizes the entropy of the induced state visitation frequencies~\citep{hazan2019maxent,mutti2020intrinsically}. One can straightforwardly extend the objective to multiple environments by considering the expectation over the class of CMPs,
$
    \mathcal{E}_{\vcmp} (\pi) = \EV_{\substack{\cmp \sim p_{\vcmp} \\ \tau \sim p_{\pi, \cmp}}} \big[ H_{\tau} \big],
    %\label{eq:exp_risk_neutral}
$
where the usual entropy objective over the single environment $\cmp_i$ can be easily recovered by setting $p_{\cmp_i} = 1$.
However, this objective function does not account for the tail behavior of $H_\tau$, \ie for the performance in environments of $\vcmp$ that are rare or particularly unfavorable. This is decidedly undesirable as the agent may be tasked with an MDP built upon one of these adverse environments in the subsequent supervised fine-tuning, where even an optimal strategy \wrt $\mathcal{E}_{\vcmp} (\pi)$ may fail to provide sufficient exploration. To overcome this limitation, we look for a more nuanced exploration objective that balances the expected performance with the sensitivity to the tail behavior. By taking inspiration from the risk-averse optimization literature~\citep{rockafellar2000optimization}, we consider the CVaR of the state visitation entropy induced by $\pi$ over $\vcmp$,
\begin{equation}
\begin{aligned}
    \mathcal{E}_{\vcmp}^\alpha (\pi) &= \cvar (H_\tau) \\
    &= \EV_{\substack{\cmp \sim p_{\vcmp} \\ \tau \sim p_{\pi, \cmp}}} \big[ H_{\tau} \;|\; H_{\tau} \leq \var (H_\tau) \big],
    \label{eq:exp_risk_averse}    
\end{aligned}
\end{equation}
where $\alpha$ is a confidence level and $\mathcal{E}_{\vcmp}^1 (\pi) := \mathcal{E}_{\vcmp} (\pi)$. The lower we set the value of $\alpha$, the more we hedge against the possibility of a bad exploration outcome in some $\cmp \in \vcmp$.
In the following sections, we propose a method to effectively learn a policy $\pi^*_{\mathcal{E}} \in \arg \max \mathcal{E}_{\vcmp}^\alpha (\pi)$ through mere interactions with $\vcmp$, and we show how this serves as a pre-training for RL (the full process is depicted in Figure~\ref{fig:problem}). A preliminary theoretical characterization of the problem of optimizing $\mathcal{E}_{\vcmp}^{\alpha} (\pi)$ is provided in Appendix~\ref{sec:theory}.

\section{A Policy Gradient Approach}
\label{sec:method}

In this section, we present an algorithm, called \emph{$\alpha$-sensitive Maximum Entropy POLicy optimization} ($\alpha$MEPOL), to optimize the exploration objective in~\eqref{eq:exp_risk_averse} through mediated interactions with a class of continuous environments. 

$\alpha$MEPOL operates as a typical policy gradient approach \citep{deisenroth2013survey}. It directly searches for an optimal policy by navigating a set of parametric differentiable policies $\Pi_{\Theta} := \lbrace \pi_{\vtheta} : \vtheta \in \Theta \subseteq \reals^n \rbrace$. It does so by repeatedly updating the parameters $\vtheta$ in the gradient direction, until a stationary point is reached. This update has the form
\begin{equation*}
    \vtheta ' = \vtheta + \beta \nabla_{\vtheta} \mathcal{E}_{\vcmp}^{\alpha} (\pi_{\vtheta}),
\end{equation*}
where $\beta$ is a learning rate, and $\nabla_{\vtheta} \mathcal{E}_{\vcmp}^{\alpha} (\pi_{\vtheta})$ is the gradient of~\eqref{eq:exp_risk_averse} \wrt $\vtheta$. The following proposition provides the formula of $\nabla_{\vtheta} \mathcal{E}_{\vcmp}^{\alpha} (\pi_{\vtheta})$. The derivation follows closely the one in~\citep[][Proposition 1]{tamar2015optimizing}, which we have adapted to our objective function of interest~\eqref{eq:exp_risk_averse}.
\begin{restatable}[]{prop}{cvarPolicyGradient}
The policy gradient of the exploration objective $\mathcal{E}_{\vcmp}^{\alpha} (\pi_{\vtheta})$ \wrt $\vtheta$ is given by
\begin{equation*}
\begin{aligned}
    &\nabla_{\vtheta} \mathcal{E}_{\vcmp}^{\alpha} (\pi_{\vtheta}) =
    \EV_{\substack{\cmp \sim p_{\vcmp} \\ \tau \sim p_{\pi_{\vtheta},\cmp}}} 
    \bigg[ \bigg( \sum_{t = 0}^{T - 1} \nabla_{\vtheta} \log \pi_{\vtheta} (a_{t,\tau}|s_{t,\tau}) \bigg)  \\
    &\times \bigg( H_{\tau} - \var (H_\tau) \bigg) \bigg| H_{\tau} \leq \var (H_\tau) \bigg].
\end{aligned}
\end{equation*}
\label{prop:cvar_policy_gradient}
\end{restatable}
However, in this work we do not assume full knowledge of the class of CMPs $\vcmp$, and the expected value in Proposition~\ref{prop:cvar_policy_gradient} cannot be computed without having access to $p_{\vcmp}$ and $p_{\pi_{\vtheta}, \cmp}$. Instead, $\alpha$MEPOL computes the policy update via a Monte Carlo estimation of $\nabla_{\vtheta} \mathcal{E}_{\vcmp}^{\alpha}$ from the sampled interactions $\lbrace (\cmp_i, \tau_i) \rbrace_{i = 1}^N$ with the class of environments $\vcmp$. The policy gradient estimate itself relies on a Monte Carlo estimate of each entropy value $H_{\tau_i}$ from $\tau_i$, and a Monte Carlo estimate of $\var (H_\tau)$ given the estimated $\lbrace H_{\tau_i} \rbrace_{i = 1}^N$. The following paragraphs describe how these estimates are carried out, while Algorithm~\ref{alg:memento} provides the pseudocode of $\alpha$MEPOL. Additional details and implementation choices can be found in Appendix~\ref{apx:algorithm}.
\begin{algorithm}[t]
\caption{$\alpha$MEPOL}
\label{alg:memento}
    \textbf{Input}: percentile $\alpha$, learning rate $\beta$\\
    \textbf{Output}: policy $\pi_{\bm{\theta}}$
    \begin{algorithmic}[1] %[1] enables line numbers
        \STATE initialize $\vtheta$
        \FOR{epoch $= 0,1,\ldots$, until convergence}
        \FOR{$i = 1,2,\ldots, N$}
        \STATE{sample an environment $\mathcal{M}_i \sim p_{\vcmp}$}
        \STATE sample a trajectory $\tau_i \sim p_{\pi_{\vtheta}, \cmp_i}$
        \STATE estimate $H_{\tau_i}$ with~\eqref{eq:knn_entropy_estimator}
        \ENDFOR
        \STATE estimate $\var (H_\tau)$ with~\eqref{eq:var_estimator} 
        \STATE estimate $\nabla_{\vtheta} \mathcal{E}_{\vcmp}^{\alpha} (\pi_{\vtheta})$ with~\eqref{eq:policy_gradient_estimator}
        \STATE update parameters $\vtheta \gets \vtheta + \beta \widehat{\nabla}_{\vtheta} \mathcal{E}_{\cmp}^{\alpha} (\pi_{\vtheta})$
        \ENDFOR
    \end{algorithmic}
\end{algorithm}

\paragraph{Entropy Estimation}
We would like to compute the entropy $H_{\tau_i}$ of the state visitation frequencies $d_{\tau_i}$ from a single realization $\lbrace s_{t,\tau_i} \rbrace_{t=0}^{T - 1} \subset \tau_i$. This estimation is notoriously challenging when the state space is continuous and high-dimensional $\sspace \subseteq \reals^p$. Taking inspiration from recent works pursuing the MSVE objective~\citep{mutti2020policy,anonymous2021unsupervised,seo2021state}, we employ a principled $k$-Nearest Neighbors ($k$-NN) entropy estimator~\citep{singh2003nearest} of the form
\begin{equation}
    \widehat{H}_{\tau_i} \propto - \frac{1}{T} \sum_{t = 0}^{T - 1}
    \log \frac{k \ \Gamma (\frac{p}{2} + 1)}{T \ \big\| s_{t,\tau_i} - s^{k\text{-NN}}_{t, \tau_i} \big\|^p \ \pi^{\frac{p}{2}}},
    \label{eq:knn_entropy_estimator}
\end{equation}
where $\Gamma$ is the Gamma function, $\|\cdot\|$ is the Euclidean distance, and $s_{t, \tau_i}^{k\text{-NN}} \in \tau_i$ is the $k$-nearest neighbor of $s_{t,\tau_i}$. The intuition behind the estimator in $\eqref{eq:knn_entropy_estimator}$ is simple: We can suppose the state visitation frequencies $d_{\tau_i}$ to have a high entropy as long as the average distance between any encountered state and its $k$-NN is large. Despite its simplicity, a Euclidean metric suffices to get reliable entropy estimates in continuous control domains~\citep{mutti2020policy}.

\paragraph{VaR Estimation}
The last missing piece to get a Monte Carlo estimate of the policy gradient $\nabla_{\vtheta} \mathcal{E}_{\vcmp}^\alpha$ is the value of $\var (H_\tau)$. Being $H_{[1]}, \ldots, H_{[N]}$ the order statistics out of the estimated values $\lbrace \widehat{H}_{\tau_i} \rbrace_{i = 1}^N$, we can na\"ively estimate the VaR as
\begin{equation}
    \evar (H_\tau) = H_{[\lceil \alpha N \rceil]}.
    \label{eq:var_estimator}
\end{equation}
Albeit asymptotically unbiased, the VaR estimator in~\eqref{eq:var_estimator} is known to suffer from a large variance in finite sample regimes~\citep{kolla2019concentration}, which is aggravated by the error in the upstream entropy estimates, which provide the order statistics.
This variance is mostly harmless when we use the estimate to filter out entropy values beyond the $\alpha$\%, \ie the condition $H_\tau \leq \var (H_\tau)$ in Proposition~\ref{prop:cvar_policy_gradient}. Instead, its impact is significant when we subtract it from the values within the $\alpha$\%, \ie the term $H_\tau - \var (H_\tau)$ in Proposition~\ref{prop:cvar_policy_gradient}. 
To mitigate this issue, we consider a convenient baseline $b = - \var (H_\tau)$ to be subtracted from the latter, which gives the Monte Carlo policy gradient estimator
\begin{equation}
    \widehat{\nabla}_{\vtheta} \mathcal{E}_{\vcmp}^\alpha (\pi_{\vtheta})
    = \sum_{i = 1}^N  f_{\tau_i} \ \widehat{H}_{\tau_i} \ \ind (\widehat{H}_{\tau_i} \leq \evar (H_\tau)),
    \label{eq:policy_gradient_estimator}
\end{equation}
where $f_{\tau_i} = \sum_{t = 0}^{T - 1} \nabla_{\vtheta} \log \pi_{\vtheta} (a_{t,\tau_i}|s_{t,\tau_i})$. Notably, the baseline $b$ trades off a lower estimation error for a slight additional bias in the estimation~\eqref{eq:policy_gradient_estimator}. We found that this baseline leads to empirically good results and we provide some theoretical corroboration over its benefits in Appendix~\ref{apx:algorithm_baseline}.

\section{Empirical Evaluation}
\label{sec:experiments}

We provide an extensive empirical evaluation of the proposed methodology over the two-phase learning process described in Figure~\ref{fig:problem}, which is organized as follows:
\begin{itemize} \compresslist
    \item[\ref{sec:experiments_learning_to_explore_illlustrative}]
    We show the ability of our method in pre-training an exploration policy in a class of continuous gridworlds, emphasizing the importance of the percentile sensitivity;
    \item[\ref{sec:experiments_alpha_sensitivity}]
    We discuss how the choice of the percentile of interest affects the exploration strategy;
    \item[\ref{sec:experiments_rl_illustrative}]
    We highlight the benefit that the pre-trained strategy provides to the supervised fine-tuning on the same class;
    \item[\ref{sec:experiments_scalability_class}]
    We verify the scalability of our method with the size of the class, by considering a class of 10 continuous gridworlds;
    \item[\ref{sec:experiments_scalability_dimension}]
    We verify the scalability of our method with the dimensionality of the environments, by considering a class of 29D continuous control Ant domains;
    \item[\ref{sec:experiments_scalability_vision}] We verify the scalability of our method with visual inputs, by considering a class of 147D MiniGrid domains;
    \item[\ref{sec:experiments_meta}]
    We show that the pre-trained strategy outperforms a policy meta-trained with MAML~\citep{finn2017model, gupta2018unsupervised} on the same class.
\end{itemize}
A thorough description of the experimental setting is provided in Appendix~\ref{apx:experiments}.

\begin{figure*}[t]
    \begin{subfigure}[t]{0.34\textwidth}
        \centering
        \includegraphics[scale=1, valign=t]{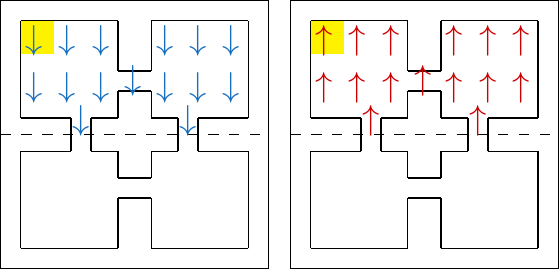}
        \vspace{5pt}
        \caption{GridWorld with Slope}
        \label{fig:gridslope_illustration}
    \end{subfigure}
    \hfill
    \begin{subfigure}[t]{0.21\textwidth}
        \centering
        \includegraphics[scale=1, valign=t]{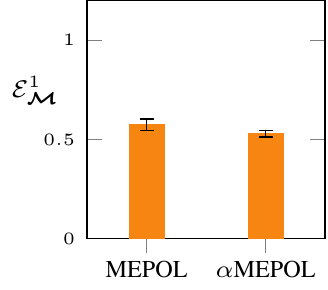}
        \vspace{-0.2cm}
        \caption{$p_{\vcmp} = [0.8, 0.2]$}
        \label{fig:gridslope_exploration_all}
    \end{subfigure}
    %\hfill
        \begin{subfigure}[t]{0.21\textwidth}
        \centering
        \includegraphics[scale=1, valign=t]{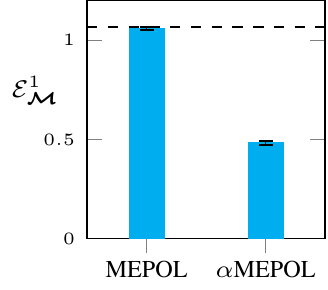}
        \vspace{-0.2cm}
        \caption{$p_{\vcmp} = [1, 0]$}
        \label{fig:gridslope_exploration_south}
    \end{subfigure}
    %\hfill
    \begin{subfigure}[t]{0.22\textwidth}
        \centering
        \includegraphics[scale=1, valign=t]{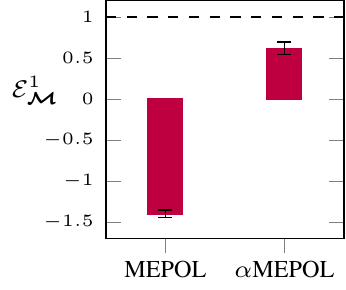}
        \vspace{-0.2cm}
        \caption{$p_{\vcmp} = [0, 1]$}
        \label{fig:gridslope_exploration_north}
    \end{subfigure}
    
    \vspace{0.2cm}
    \begin{subfigure}[t]{0.53\textwidth}
        \centering
        \includegraphics[scale=1, valign=t]{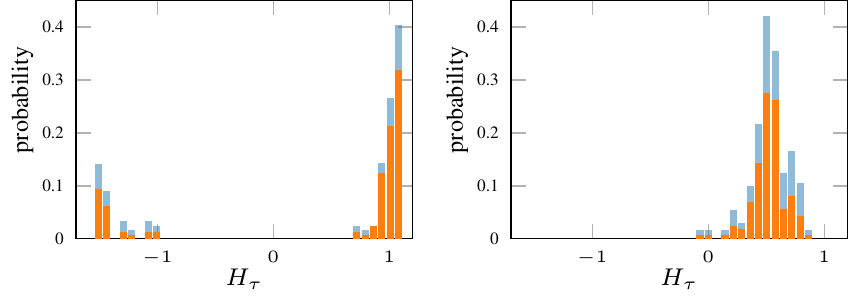}
        \vspace{-0.2cm}
        \caption{MEPOL (left) and $\alpha$MEPOL (right)}
        \label{fig:gridslope_empirical_distribution}
    \end{subfigure}
    %\hfill
    \begin{subfigure}[t]{0.47\textwidth}
        \centering
        \includegraphics[scale=1, valign=t]{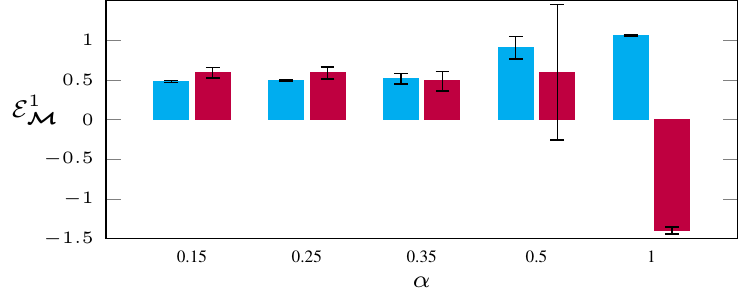}
        \vspace{-0.2cm}
        \caption{$\alpha$ sensitivity of $\alpha$MEPOL}
        \label{fig:gridslope_alpha}
    \end{subfigure}
%    \begin{subfigure}[t]{0.21\textwidth}
%        \centering
%        \includegraphics[scale=1, valign=t]{contents/plot_files/img/gridslope_alpha.pdf}
%        \vspace{-0.2cm}
%        \caption{$\alpha$ sensitivity}
%        \label{fig:gridslope_alpha}
%    \end{subfigure}
%    %\hfill
%    \begin{subfigure}[t]{0.235\textwidth}
%        \centering
%        \includegraphics[scale=1, valign=t]{contents/plot_files/img/gridslope_alpha_composition.pdf}
%        \vspace{-0.2cm}
%        \caption{$\alpha$\% composition}
%        \label{fig:gridslope_alpha_composition}
%    \end{subfigure}
    %\vspace{-0.2cm}
    \caption{Pre-training performance $\mathcal{E}_{\vcmp}^1$ obtained by $\alpha$MEPOL ($\alpha = 0.2$) and MEPOL in the \emph{GridWorld with Slope} domain \textbf{(a)}. The polices are trained on  \textbf{(b)} and tested on \textbf{(b, c, d)}. The dashed lines in \textbf{(c, d)} represent the optimal performance. The empirical distribution having mean in \textbf{(b)} is reported in \textbf{(e)}. The behaviour of $\alpha$MEPOL with different $\alpha$ is reported in \textbf{(f)}. For every plot, we provide 95\% c.i. over 10 runs.}
    \label{fig:gridslope_exploration}
\end{figure*}
\subsection{Unsupervised Pre-Training with Percentile Sensitivity}
\label{sec:experiments_learning_to_explore_illlustrative}
We consider a class $\vcmp$ composed of two different configurations of a continuous gridworld domain with 2D states and 2D actions, which we call the \emph{GridWorld with Slope}. In each configuration, the agent navigates through four rooms connected by narrow hallways, by choosing a (bounded) increment along the coordinate directions. A visual representation of the setting can be found in Figure~\ref{fig:gridslope_illustration}, where the shaded areas denote the initial state distribution and the arrows render a slope that favors or contrasts the agent's movement. The configuration on the left has a south-facing slope, and thus it is called GridWorld with South slope (GWS). Instead, the one on the right is called GridWorld with North slope (GWN) as it has a north-facing slope. This class of environments is unbalanced (and thus interesting to our purpose) for two reasons: First, the GWN configuration is more challenging from a pure exploration standpoint, since the slope prevents the agent from easily reaching the two bottom rooms; secondly, the distribution over the class is also unbalanced, as it is $p_{\vcmp} = [Pr(\text{GWS}), Pr(\text{GWN})] = [0.8, 0.2]$.
In this setting, we compare  $\alpha$MEPOL against MEPOL~\cite{mutti2020policy}, which is akin to $\alpha$MEPOL with $\alpha = 1$,\footnote{The pseudocode is identical to Algorithm~\ref{alg:memento} except that all trajectories affect the gradient estimate in~\eqref{eq:policy_gradient_estimator}.} to highlight the importance of percentile sensitivity \wrt a na\"ive approach to the multiple-environments scenario. The methods are evaluated in terms of the state visitation entropy $\mathcal{E}_{\vcmp}^1$ induced by the exploration strategies they learn.

In Figure~\ref{fig:gridslope_exploration}, we compare the performance of the optimal exploration strategy obtained by running $\alpha$MEPOL ($\alpha=0.2$) and MEPOL for 150 epochs on the GridWorld with Slope class ($p_{\vcmp} = [0.8, 0.2]$). We show that the two methods achieve a very similar expected performance over the class (Figure~\ref{fig:gridslope_exploration_all}). However, this expected performance is the result of a (weighted) average of very different contributions. As anticipated, MEPOL has a strong performance in GWS ($p_{\vcmp} = [1, 0]$, Figure~\ref{fig:gridslope_exploration_south}), which is close to the configuration-specific optimum (dashed line), but it displays a bad showing in the adverse GWN ($p_{\vcmp} = [0, 1]$, Figure~\ref{fig:gridslope_exploration_north}). Conversely, $\alpha$MEPOL learns a strategy that is much more robust to the configuration, showing a similar performance in GWS and GWN, as the percentile sensitivity prioritizes the worst case during training. 
To confirm this conclusion, we look at the actual distribution that is generating the expected performance in Figure~\ref{fig:gridslope_exploration_all}. In Figure~\ref{fig:gridslope_empirical_distribution}, we provide the empirical distribution of the trajectory-wise performance ($H_\tau$), considering a batch of 200 trajectories with $p_{\vcmp} = [0.8, 0.2]$. It clearly shows that MEPOL is heavy-tailed towards lower outcomes, whereas $\alpha$MEPOL concentrates around the mean. \emph{This suggests that with a conservative choice of $\alpha$ we can induce a good exploration outcome for every trajectory (and any configuration), while without percentile sensitivity we cannot hedge against the risk of particularly bad outcomes.} However, let us point out that not all classes of environments would expose such an issue for a na\"ive, risk-neutral approach (see Appendix~\ref{apx:counterexample} for a counterexample), but it is fair to assume that this would arguably generalize to any setting where there is an imbalance (either in the hardness of the configurations, or in their sampling probability) in the class. These are the settings we care about, as they require nuanced solutions (\eg $\alpha$MEPOL) for scenarios with multiple environments.
\begin{figure*}[t]
    \centering
    \begin{subfigure}[t]{0.5\textwidth}
        \centering
        \includegraphics[scale=1]{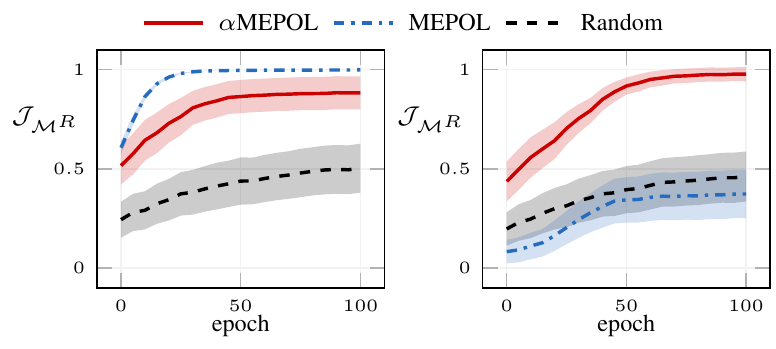}
        \vspace{-0.15cm}
        \caption{Fine-tuning on GWS (left) and GWN (right)}
        \label{fig:gridslope_rl_return}
    \end{subfigure}
    \begin{subfigure}[t]{0.3\textwidth}
        \centering
        \includegraphics[scale=1]{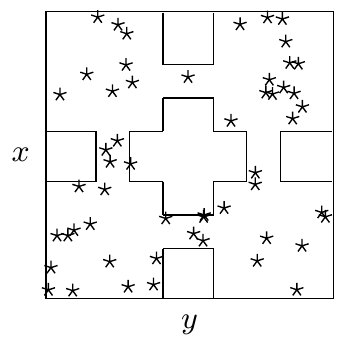}
        \vspace{-0.2cm}
        \caption{Sampled goals}
        \label{fig:gridslope_sampled_goals}
     \end{subfigure}
     %\vspace{-0.2cm}
    \caption{Fine-tuning performance $\mathcal{J}_{\cmp^R}$ as a function of learning epochs achieved by TRPO initialized with $\alpha$MEPOL ($\alpha = 0.2$), MEPOL, and random exploration strategies, when dealing with a set of RL tasks specified on the \emph{GridWorld with Slope} domain \textbf{(a)}. We provide 95\% c.i. over 50 randomly sampled goal locations \textbf{(b)}.}
    \label{fig:gridslope_rl}
\end{figure*}

\subsection{On the Value of the Percentile}
\label{sec:experiments_alpha_sensitivity}
In this section, we consider repeatedly training $\alpha$MEPOL with different values of $\alpha$ in the GridWorld with Slope domain, and we compare the resulting exploration performance $\mathcal{E}_{\vcmp}^1$ as before. In Figure~\ref{fig:gridslope_alpha}, we can see that the lower $\alpha$ we choose, the more we prioritize GWN (right bar for every $\alpha$) at the expense of GWS (left bar). Note that this trend carries on with increasing $\alpha$, ending in the values of Figures~\ref{fig:gridslope_exploration_south},~\ref{fig:gridslope_exploration_north}. The reason for this behavior is quite straightforward, the smaller is $\alpha$, the larger is the share of trajectories from the adverse configuration (GWN) ending up in the percentile at first, and thus the more GWN affects the policy update (see the gradient in~\eqref{eq:policy_gradient_estimator}). 
\emph{Note that the value of the percentile $\alpha$ should not be intended as a hyper-parameter to tune via trial and error, but rather as a parameter to select the desired risk profile of the algorithm.}
Indeed, there is not a way to say which of the outcomes in Figure~\ref{fig:gridslope_alpha} is preferable, as they are all reasonable trade-offs between the average and worst-case performance, which might be suited for specific applications.
For the sake of consistency, in every experiment of our analysis we report results with a value of $\alpha$ that matches the sampling probability of the worst-case configuration, but similar arguments could be made for different choices of $\alpha$.

\subsection{Supervised Fine-Tuning}
\label{sec:experiments_rl_illustrative}
To assess the benefit of the pre-trained strategy, we design a family of MDPs $\cmp^R$, where $\cmp \in \{\text{GWS},\text{GWN}\}$, and $R$ is any sparse reward function that gives 1 when the agent reaches the area nearby a random goal location and 0 otherwise. On this family, we compare the performance achieved by TRPO~\citep{schulman2015trust} with different initializations: The exploration strategies learned (as in Section~\ref{sec:experiments_learning_to_explore_illlustrative}) by $\alpha$MEPOL ($\alpha = 0.2$) and MEPOL, or a randomly initialized policy (Random). These three variations are evaluated in terms of their average return $\mathcal{J}_{\cmp^R}$, which is defined in~\eqref{eq:rl_objective}, over 50 randomly generated goal locations (Figure~\ref{fig:gridslope_sampled_goals}).
As expected, the performance of TRPO with MEPOL is competitive in the GWS configuration (Figure~\ref{fig:gridslope_rl}), but it falls sharply in the GWN configuration, where it is not significantly better than TRPO with Random. Instead, the performance of TRPO with $\alpha$MEPOL is strong on both GWS and GWN. Despite the simplicity of the domain, solving an RL problem in GWN with an adverse goal location is far-fetched for both a random initialization and a na\"ive solution to the problem of  unsupervised RL in multiple environments.

\begin{figure*}[t]
    \centering
    \includegraphics[scale=1]{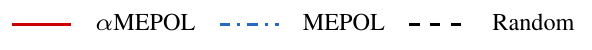}
    
    %\centering
    \begin{subfigure}[t]{0.49\textwidth}
        \centering
        \includegraphics[scale=1, valign=t]{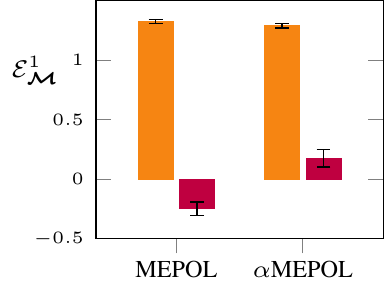}
        \includegraphics[scale=1, valign=t]{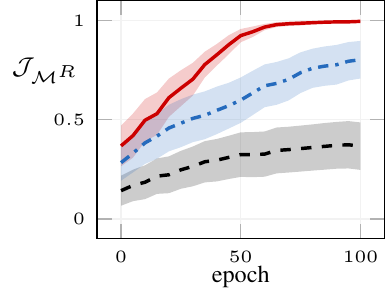}
        \vspace{-0.2cm}
        \caption{MultiGrid: Pre-training (left) and fine-tuning (right)}
        \label{fig:multigrid}
    \end{subfigure}
    \hfill
    \begin{subfigure}[t]{0.49\textwidth}
        \centering
        \includegraphics[scale=1, valign=t]{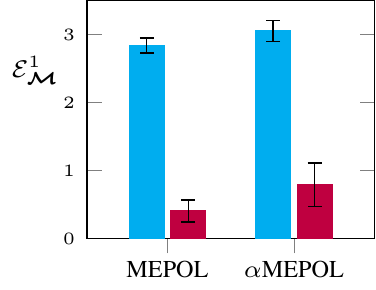}
        \includegraphics[scale=1, valign=t]{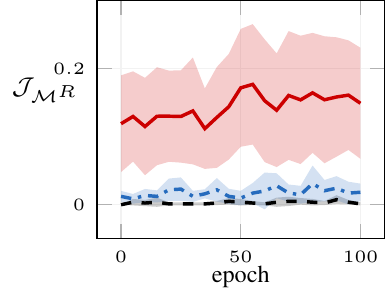}
        \vspace{-0.2cm}
        \caption{Ant: Pre-training (left) and fine-tuning (right)}
        \label{fig:ant}
    \end{subfigure}
    
    \vspace{0.2cm}
    \begin{subfigure}[t]{0.49\textwidth}
        \centering
        \hspace{0.5cm}
        \includegraphics[scale=0.132, valign=t]{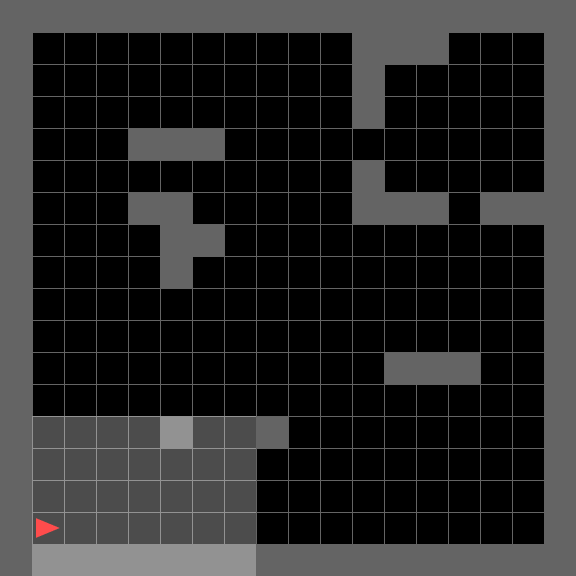}
        \hspace{0.8cm}
        \includegraphics[scale=0.235, valign=t]{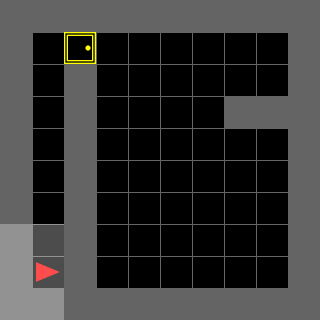}
        \vspace{0.03cm}
        \caption{MiniGrid: EasyG (left) and AdvG (right)}
        \label{fig:minigrid_visualization}
    \end{subfigure}
    \hfill
    \begin{subfigure}[t]{0.49\textwidth}
        \centering
        \includegraphics[scale=1, valign=t]{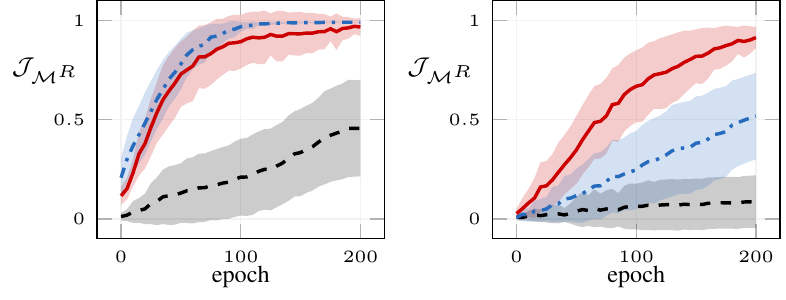}
        \vspace{-0.2cm}
        \caption{MiniGrid: fine-tuning on EasyG (left) and AdvG (right)}
        \label{fig:minigrid}
    \end{subfigure}
    %\vspace{-0.2cm}
    \caption{Pre-training performance $\mathcal{E}^1_{\vcmp}$ (95\% c.i. over 10 runs) achieved by $\alpha$MEPOL ($\alpha=0.1$ \textbf{(a)}, $\alpha = 0.2$ \textbf{(b)}) and MEPOL in the in the \emph{MultiGrid} \textbf{(a)} and \emph{Ant} \textbf{(b)} domains.
    Fine-tuning performance $\mathcal{J}_{\cmp^R}$ (95\% c.i. over 50 tasks \textbf{(a)}, 8 tasks \textbf{(b)}, 13 tasks \textbf{(d)}) obtained by TRPO with corresponding initialization ($\alpha$MEPOL, MEPOL, Random), in the \emph{MultiGrid} \textbf{(a)}, \emph{Ant} \textbf{(b)}, and \emph{MiniGrid} \textbf{(d)} domains. \emph{MiniGrid} domains are illustrated in \textbf{(c)}.}
\end{figure*}
\subsection{Scaling to Larger Classes of Environments}
\label{sec:experiments_scalability_class}
In this section, we consider a class $\vcmp$ composed of 10 different configurations of the continuous gridworlds presented in Section~\ref{sec:experiments_learning_to_explore_illlustrative} (including the GWN as the worst-case configuration) which we call the \emph{MultiGrid} domain. As before, we compare $\alpha$MEPOL ($\alpha = 0.1$) and MEPOL on the exploration performance $\mathcal{E}_{\vcmp}^{1}$ achieved by the optimal strategy, in this case considering a uniformly distributed $p_{\vcmp}$. While the average performance of MEPOL is slightly higher across the class (Figure~\ref{fig:multigrid} left, left bar), $\alpha$MEPOL still has a decisive advantage in the worst-case configuration (Figure~\ref{fig:multigrid} left, right bar). Just as in Section~\ref{sec:experiments_rl_illustrative}, this advantage transfer to the fine-tuning, where we compare the average return $\mathcal{J}_{\cmp^R}$ achieved by TRPO with $\alpha$MEPOL, MEPOL, and Random initializations over 50 random goal locations in the GWN configuration (Figure~\ref{fig:multigrid} right). \emph{Whereas in the following sections we will only consider classes of two environments, this experiment shows that the arguments made for small classes of environments can easily generalize to larger classes.}

\subsection{Scaling to Increasing Dimensions}
\label{sec:experiments_scalability_dimension}
In this section, we consider a class $\vcmp$ consisting of two Ant environments, with 29D states and 8D actions. In the first, sampled with probability $p_{\cmp_1} = 0.8$, the Ant faces a wide descending staircase (\emph{Ant Stairs Down}). In the second, the Ant faces a narrow ascending staircase (\emph{Ant Stairs Up}, sampled with probability $p_{\cmp_2} = 0.2$), which is significantly harder to explore than the former. In the mold of the gridworlds in Section~\ref{sec:experiments_learning_to_explore_illlustrative}, these two configurations are specifically designed to create an imbalance in the class. As in Section~\ref{sec:experiments_learning_to_explore_illlustrative}, we compare $\alpha$MEPOL ($\alpha = 0.2$) against MEPOL on the exploration performance $\mathcal{E}_{\vcmp}^{1}$ achieved after 500 epochs. $\alpha$MEPOL fares slightly better than MEPOL both in the worst-case configuration (Figure~\ref{fig:ant} left, right bar) and, surprisingly, in the easier one (Figure~\ref{fig:ant} left, left bar).\footnote{Note that this would not happen in general, as we expect $\alpha$MEPOL to be better in the worst-case but worse on average. In this setting, the percentile sensitivity positively biases the average performance due to the peculiar structure of the environments.}
Then, we design a set of incrementally challenging fine-tuning tasks in the \emph{Ant Stairs Up}, which give reward 1 upon reaching a certain step of the staircase. Also in this setting, TRPO with $\alpha$MEPOL initialization outperforms TRPO with MEPOL and Random in terms of the average return $\mathcal{J}_{\cmp^R}$ (Figure~\ref{fig:ant} right). Note that these sparse-reward continuous control tasks are particularly arduous: TRPO with MEPOL and Random barely learns anything, while even TRPO with $\alpha$MEPOL does not handily reach the optimal average return (1). 

\begin{figure*}[t]
    \centering
    \includegraphics[scale=1]{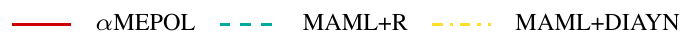}
    
    \centering
    \begin{subfigure}[t]{0.49\textwidth}
        \centering
        \includegraphics[scale=1, valign=t]{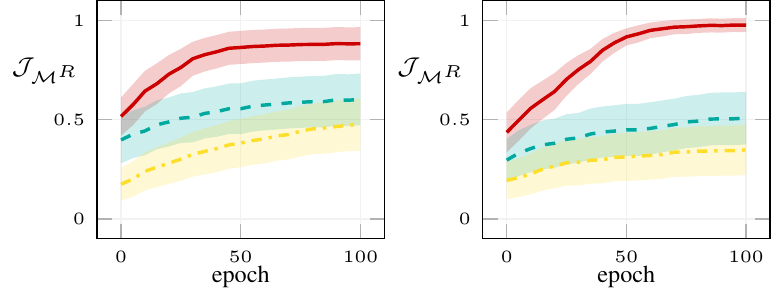}
        \vspace{-0.2cm}
        \caption{GridWorld with Slope: GWS (left) and GWN (right)}
        \label{fig:maml_gridslope}
    \end{subfigure}
    \hspace{0.5cm}
    \begin{subfigure}[t]{0.25\textwidth}
        \centering
        \includegraphics[scale=1, valign=t]{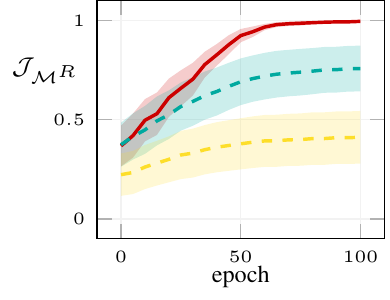}
        \vspace{-0.2cm}
        \caption{MultiGrid}
        \label{fig:maml_multigrid}
    \end{subfigure}
    % \begin{subfigure}[t]{0.24\textwidth}
    %     \centering
    %     \includegraphics[scale=0.85, valign=t]{contents/plot_files/img/maml_gridslope_adaptation.pdf}
    %     \vspace{-0.2cm}
    %     \caption{Fast Adaptation}
    %     \label{fig:maml_adaptation}
    % \end{subfigure}
    %\vspace{-0.2cm}
    \caption{Fine-tuning performance $\mathcal{J}_{\cmp^R}$ achieved by TRPO initialized with $\alpha$MEPOL ($\alpha = 0.2$ \textbf{(a)}, $\alpha=0.1$ \textbf{(b)}), a MAML+R meta-policy, and a MAML+DIAYN meta-policy, when dealing with a set of RL tasks in the \emph{GridWorld with Slope} \textbf{(a)} and the \emph{MultiGrid} \textbf{(b)} domains. We provide 95\% c.i. over 50 tasks. 
    %We illustrate the fast-adapting behavior of the MAML policy in \textbf{(c)}.
    }
    \label{fig:maml}
\end{figure*}
\subsection{Scaling to Visual Inputs}
\label{sec:experiments_scalability_vision}
In this section, we consider a class $\vcmp$ of two partially-observable MiniGrid~\citep{chevalier2018minimalistic} environments, in which the observation is a 147D image of the agent's field of view. 
In Figure~\ref{fig:minigrid_visualization}, we provide a visualization of the domain: The easier configuration (EasyG, left) is sampled with probability $p_{\cmp_1} = 0.8$, the adverse configuration (AdvG, right) is sampled with probability $p_{\cmp_2} = 0.2$. Two factors make the AdvG more challenging to explore, which are the presence of a door at the top-left of the grid, and reversing the effect of agent's movements (\eg the agent goes backward when it tries to go forward). 
Whereas in all the previous experiments we estimated the entropy on the raw input features, visual inputs require a wiser choice of a metric. As proposed in~\citep{seo2021state}, we process the observations through a random encoder before computing the entropy estimate in~\eqref{eq:knn_entropy_estimator}, while keeping everything else as in Algorithm~\ref{alg:memento}. We run this slightly modified version of $\alpha$MEPOL ($\alpha = 0.2$) and MEPOL for 300 epochs. Then, we compare TRPO with the learned initializations (as well as Random) on sparse-reward fine-tuning tasks defined upon the class. As in previous settings, TRPO with $\alpha$MEPOL results slightly worse than TRPO with MEPOL in the easier configuration (Figure~\ref{fig:minigrid}, left), but significantly better in the worst-case (Figure~\ref{fig:minigrid}, right). Notably, TRPO from scratch struggles to learn the tasks, especially in the AdvG (Figure~\ref{fig:minigrid}, right).
\emph{Although the MiniGrid domain is extremely simple from a vision standpoint, we note that the same architecture can be employed in more challenging scenarios~\citep{seo2021state}, while the focus of this experiment is the combination between visual inputs and multiple environments.}

\subsection{Comparison with Meta-RL}
\label{sec:experiments_meta}
In this section, we compare our approach against meta-training a policy with MAML~\citep{finn2017model} on the same \emph{GridWorld with Slope} ($p_{\vcmp} = [0.8, 0.2]$) and \emph{MultiGrid} (uniformly distributed $p_{\vcmp}$) domains that we have previously presented. 
Especially, we consider two relevant baselines. The first is MAML+R, to which we provide full access to the tasks (\ie rewards) during meta-training. Note that this gives MAML+R an edge over $\alpha$MEPOL, which operates reward-free training. The second is MAML+DIAYN~\citep{gupta2018unsupervised}, which operates unsupervised meta-training through an intrinsic reward function learned with DIAYN~\citep{eysenbach2018diversity}.
As in previous sections, we consider the average return $\mathcal{J}_{\cmp^R}$ achieved by TRPO initialized with the exploration strategy learned by $\alpha$MEPOL or the meta-policy learned by MAML+R and MAML+DIAYN. TRPO with $\alpha$MEPOL fares clearly better than TRPO with the meta-policies in all the configurations (Figures~\ref{fig:maml_gridslope}, \ref{fig:maml_multigrid}). Even if it works fine in fast adaptation (see Appendix~\ref{apx:meta_experiments_details}), \emph{MAML struggles to encode the diversity of task distribution into a single meta-policy and to deal with the most adverse tasks in the long run.} Moreover, DIAYN does not specifically handle multiple environments, and it fails to cope with the larger \emph{MultiGrid} class.

\section{Conclusions}

In this paper, we addressed the problem of unsupervised RL in a class of multiple environments. First, we formulated the problem within a tractable objective, which is inspired by MSVE but includes an additional percentile sensitivity. Then, we presented a policy gradient algorithm, $\alpha$MEPOL, to optimize this objective. Finally, we provided an extensive experimental analysis to show its ability in the unsupervised pre-training and the benefits it brings to the subsequent supervised fine-tuning.
We believe that this paper motivates the importance of designing specific solutions to the relevant problem of unsupervised RL in multiple environments.

%\nocite{*}
\bibliography{biblio}

\begin{thebibliography}{56}
\providecommand{\natexlab}[1]{#1}

\bibitem[{Achiam et~al.(2018)Achiam, Edwards, Amodei, and
  Abbeel}]{achiam2018variational}
Achiam, J.; Edwards, H.; Amodei, D.; and Abbeel, P. 2018.
\newblock Variational option discovery algorithms.
\newblock \emph{arXiv preprint arXiv:1807.10299}.

\bibitem[{Ajgl and {\v{S}}imandl(2011)}]{ajgl2011differential}
Ajgl, J.; and {\v{S}}imandl, M. 2011.
\newblock Differential entropy estimation by particles.
\newblock \emph{IFAC Proceedings Volumes}.

\bibitem[{Bellemare et~al.(2016)Bellemare, Srinivasan, Ostrovski, Schaul,
  Saxton, and Munos}]{bellemare2016unifying}
Bellemare, M.; Srinivasan, S.; Ostrovski, G.; Schaul, T.; Saxton, D.; and
  Munos, R. 2016.
\newblock Unifying count-based exploration and intrinsic motivation.
\newblock \emph{Advances in Neural Information Processing Systems}.

\bibitem[{Burda et~al.(2019)Burda, Edwards, Storkey, and Klimov}]{burda2019rnd}
Burda, Y.; Edwards, H.; Storkey, A.; and Klimov, O. 2019.
\newblock Exploration by random network distillation.
\newblock In \emph{International Conference on Learning Representations}.

\bibitem[{Campos et~al.(2021)Campos, Sprechmann, Hansen, Barreto, Kapturowski,
  Vitvitskyi, Badia, and Blundell}]{campos2021finetuning}
Campos, V.; Sprechmann, P.; Hansen, S.; Barreto, A.; Kapturowski, S.;
  Vitvitskyi, A.; Badia, A.~P.; and Blundell, C. 2021.
\newblock Beyond fine-tuning: Transferring behavior in reinforcement learning.
\newblock \emph{arXiv preprint arXiv:2102.13515}.

\bibitem[{Chevalier-Boisvert, Willems, and
  Pal(2018)}]{chevalier2018minimalistic}
Chevalier-Boisvert, M.; Willems, L.; and Pal, S. 2018.
\newblock Minimalistic gridworld environment for openai gym.
\newblock \emph{GitHub repository}.

\bibitem[{Chow and Ghavamzadeh(2014)}]{chow2014algorithms}
Chow, Y.; and Ghavamzadeh, M. 2014.
\newblock Algorithms for CVaR optimization in MDPs.
\newblock In \emph{Advances in neural information processing systems}.

\bibitem[{Csisz{\'a}r and Talata(2006)}]{csiszar2006context}
Csisz{\'a}r, I.; and Talata, Z. 2006.
\newblock Context tree estimation for not necessarily finite memory processes,
  via BIC and MDL.
\newblock \emph{IEEE Transactions on Information theory}.

\bibitem[{Deisenroth, Neumann, and Peters(2013)}]{deisenroth2013survey}
Deisenroth, M.; Neumann, G.; and Peters, J. 2013.
\newblock A survey on policy search for robotics.
\newblock \emph{Foundations and Trends in Robotics}.

\bibitem[{Duan et~al.(2016)Duan, Chen, Houthooft, Schulman, and
  Abbeel}]{duan2016benchmarking}
Duan, Y.; Chen, X.; Houthooft, R.; Schulman, J.; and Abbeel, P. 2016.
\newblock Benchmarking deep reinforcement learning for continuous control.
\newblock In \emph{Proceedings of the International Conference on Machine
  Learning}.

\bibitem[{Eysenbach et~al.(2018)Eysenbach, Gupta, Ibarz, and
  Levine}]{eysenbach2018diversity}
Eysenbach, B.; Gupta, A.; Ibarz, J.; and Levine, S. 2018.
\newblock Diversity is all you need: Learning skills without a reward function.
\newblock In \emph{International Conference on Learning Representations}.

\bibitem[{Finn, Abbeel, and Levine(2017)}]{finn2017model}
Finn, C.; Abbeel, P.; and Levine, S. 2017.
\newblock Model-agnostic meta-learning for fast adaptation of deep networks.
\newblock In \emph{Proceedings of the International Conference on Machine
  Learning}.

\bibitem[{Gregor, Rezende, and Wierstra(2016)}]{gregor2016variational}
Gregor, K.; Rezende, D.~J.; and Wierstra, D. 2016.
\newblock Variational intrinsic control.
\newblock \emph{arXiv preprint arXiv:1611.07507}.

\bibitem[{Guo et~al.(2021)Guo, Azar, Saade, Thakoor, Piot, Pires, Valko,
  Mesnard, Lattimore, and Munos}]{guo2021geometric}
Guo, Z.~D.; Azar, M.~G.; Saade, A.; Thakoor, S.; Piot, B.; Pires, B.~A.; Valko,
  M.; Mesnard, T.; Lattimore, T.; and Munos, R. 2021.
\newblock Geometric entropic exploration.
\newblock \emph{arXiv preprint arXiv:2101.02055}.

\bibitem[{Gupta et~al.(2018{\natexlab{a}})Gupta, Eysenbach, Finn, and
  Levine}]{gupta2018unsupervised}
Gupta, A.; Eysenbach, B.; Finn, C.; and Levine, S. 2018{\natexlab{a}}.
\newblock Unsupervised meta-learning for reinforcement learning.
\newblock \emph{arXiv preprint arXiv:1806.04640}.

\bibitem[{Gupta et~al.(2018{\natexlab{b}})Gupta, Mendonca, Liu, Abbeel, and
  Levine}]{gupta2018maesn}
Gupta, A.; Mendonca, R.; Liu, Y.; Abbeel, P.; and Levine, S.
  2018{\natexlab{b}}.
\newblock Meta-reinforcement learning of structured exploration strategies.
\newblock In \emph{Advances in Neural Information Processing Systems}.

\bibitem[{Hallak, Di~Castro, and Mannor(2015)}]{hallak2015contextual}
Hallak, A.; Di~Castro, D.; and Mannor, S. 2015.
\newblock Contextual Markov decision processes.
\newblock \emph{arXiv preprint arXiv:1502.02259}.

\bibitem[{Hazan et~al.(2019)Hazan, Kakade, Singh, and
  Van~Soest}]{hazan2019maxent}
Hazan, E.; Kakade, S.; Singh, K.; and Van~Soest, A. 2019.
\newblock Provably efficient maximum entropy exploration.
\newblock In \emph{Proceedings of the International Conference on Machine
  Learning}.

\bibitem[{Jin et~al.(2020)Jin, Krishnamurthy, Simchowitz, and
  Yu}]{jin2020rewardfree}
Jin, C.; Krishnamurthy, A.; Simchowitz, M.; and Yu, T. 2020.
\newblock Reward-free exploration for reinforcement learning.
\newblock In \emph{Proceedings of the International Conference on Machine
  Learning}.

\bibitem[{Kolla et~al.(2019)Kolla, Prashanth, Bhat, and
  Jagannathan}]{kolla2019concentration}
Kolla, R.~K.; Prashanth, L.; Bhat, S.~P.; and Jagannathan, K. 2019.
\newblock Concentration bounds for empirical conditional value-at-risk: The
  unbounded case.
\newblock \emph{Operations research letters}.

\bibitem[{Kwon et~al.(2021)Kwon, Efroni, Caramanis, and Mannor}]{kwon2021rl}
Kwon, J.; Efroni, Y.; Caramanis, C.; and Mannor, S. 2021.
\newblock RL for latent MDPs: Regret guarantees and a lower bound.
\newblock In \emph{Advances in Neural Information Processing Systems}.

\bibitem[{L.A., Jagannathan, and Kolla(2020)}]{kolla2020concentration}
L.A., P.; Jagannathan, K.; and Kolla, R. 2020.
\newblock Concentration bounds for {CV}a{R} estimation: The cases of
  light-tailed and heavy-tailed distributions.
\newblock In \emph{Proceedings of the International Conference on Machine
  Learning}.

\bibitem[{Laskin et~al.(2021)Laskin, Yarats, Liu, Lee, Zhan, Lu, Cang, Pinto,
  and Abbeel}]{laskin2021urlb}
Laskin, M.; Yarats, D.; Liu, H.; Lee, K.; Zhan, A.; Lu, K.; Cang, C.; Pinto,
  L.; and Abbeel, P. 2021.
\newblock URLB: Unsupervised reinforcement learning benchmark.
\newblock In \emph{Conference on Neural Information Processing Systems Datasets
  and Benchmarks Track (Round 2)}.

\bibitem[{Lee et~al.(2019)Lee, Eysenbach, Parisotto, Xing, Levine, and
  Salakhutdinov}]{lee2019smm}
Lee, L.; Eysenbach, B.; Parisotto, E.; Xing, E.; Levine, S.; and Salakhutdinov,
  R. 2019.
\newblock Efficient exploration via state marginal matching.
\newblock \emph{arXiv preprint arXiv:1906.05274}.

\bibitem[{Liu and Abbeel(2021{\natexlab{a}})}]{liu2021aps}
Liu, H.; and Abbeel, P. 2021{\natexlab{a}}.
\newblock Aps: Active pretraining with successor features.
\newblock In \emph{Proceedings of the International Conference on Machine
  Learning}.

\bibitem[{Liu and Abbeel(2021{\natexlab{b}})}]{anonymous2021unsupervised}
Liu, H.; and Abbeel, P. 2021{\natexlab{b}}.
\newblock Behavior from the void: Unsupervised active pre-training.
\newblock In \emph{Advances in Neural Information Porcessing Systems}.

\bibitem[{Metelli, Mutti, and Restelli(2018)}]{metelli2018configurable}
Metelli, A.~M.; Mutti, M.; and Restelli, M. 2018.
\newblock Configurable Markov decision processes.
\newblock In \emph{Proceedings of the International Conference on Machine
  Learning}.

\bibitem[{Metelli et~al.(2018)Metelli, Papini, Faccio, and
  Restelli}]{metelli2018policy}
Metelli, A.~M.; Papini, M.; Faccio, F.; and Restelli, M. 2018.
\newblock Policy optimization via importance sampling.
\newblock In \emph{Advances in Neural Information Processing Systems}.

\bibitem[{Mutti, Pratissoli, and Restelli(2021)}]{mutti2020policy}
Mutti, M.; Pratissoli, L.; and Restelli, M. 2021.
\newblock Task-agnostic exploration via policy gradient of a non-parametric
  state entropy estimate.
\newblock In \emph{Proceedings of the AAAI Conference on Artificial
  Intelligence}.

\bibitem[{Mutti and Restelli(2020)}]{mutti2020intrinsically}
Mutti, M.; and Restelli, M. 2020.
\newblock An intrinsically-motivated approach for learning highly exploring and
  fast mixing policies.
\newblock In \emph{Proceedings of the AAAI Conference on Artificial
  Intelligence}.

\bibitem[{Ostrovski et~al.(2017)Ostrovski, Bellemare, Oord, and
  Munos}]{ostrovski2017count}
Ostrovski, G.; Bellemare, M.~G.; Oord, A.; and Munos, R. 2017.
\newblock Count-based exploration with neural density models.
\newblock In \emph{International Conference on Machine Learning}.

\bibitem[{Parisi et~al.(2021)Parisi, Dean, Pathak, and
  Gupta}]{parisi2021interesting}
Parisi, S.; Dean, V.; Pathak, D.; and Gupta, A. 2021.
\newblock Interesting object, curious agent: Learning task-agnostic
  exploration.
\newblock In \emph{Advances in Neural Information Processing Systems}.

\bibitem[{Parisi, Pirotta, and Restelli(2016)}]{parisi2016multi}
Parisi, S.; Pirotta, M.; and Restelli, M. 2016.
\newblock Multi-objective reinforcement learning through continuous pareto
  manifold approximation.
\newblock \emph{Journal of Artificial Intelligence Research}.

\bibitem[{Pathak et~al.(2017)Pathak, Agrawal, Efros, and
  Darrell}]{pathak2017curiosity}
Pathak, D.; Agrawal, P.; Efros, A.~A.; and Darrell, T. 2017.
\newblock Curiosity-driven exploration by self-supervised prediction.
\newblock In \emph{Proceedings of the International Conference on Machine
  Learning}.

\bibitem[{Pathak, Gandhi, and Gupta(2019)}]{pathak2019self}
Pathak, D.; Gandhi, D.; and Gupta, A. 2019.
\newblock Self-supervised exploration via disagreement.
\newblock In \emph{Proceedings of the International Conference on Machine
  Learning}.

\bibitem[{Pirotta, Restelli, and Bascetta(2015)}]{pirotta2015policy}
Pirotta, M.; Restelli, M.; and Bascetta, L. 2015.
\newblock Policy gradient in Lipschitz Markov decision processes.
\newblock \emph{Machine Learning}.

\bibitem[{Puterman(2014)}]{puterman2014markov}
Puterman, M.~L. 2014.
\newblock \emph{Markov decision processes: Discrete stochastic dynamic
  programming}.
\newblock John Wiley \& Sons.

\bibitem[{Rajendran et~al.(2020)Rajendran, Lewis, Veeriah, Lee, and
  Singh}]{rajendran2020should}
Rajendran, J.; Lewis, R.; Veeriah, V.; Lee, H.; and Singh, S. 2020.
\newblock How should an agent practice?
\newblock In \emph{Proceedings of the AAAI Conference on Artificial
  Intelligence}.

\bibitem[{Rajeswaran et~al.(2016)Rajeswaran, Ghotra, Ravindran, and
  Levine}]{rajeswaran2016epopt}
Rajeswaran, A.; Ghotra, S.; Ravindran, B.; and Levine, S. 2016.
\newblock EPOpt: Learning robust neural network policies using model ensembles.
\newblock In \emph{Proceedings of the International Conference on Learning
  Representations}.

\bibitem[{Rockafellar, Uryasev et~al.(2000)}]{rockafellar2000optimization}
Rockafellar, R.~T.; Uryasev, S.; et~al. 2000.
\newblock Optimization of conditional value-at-risk.
\newblock \emph{Journal of risk}.

\bibitem[{Satia and Lave~Jr(1973)}]{satia1973markovian}
Satia, J.~K.; and Lave~Jr, R.~E. 1973.
\newblock Markovian decision processes with uncertain transition probabilities.
\newblock \emph{Operations Research}.

\bibitem[{Schmidhuber(1991)}]{schmidhuber1991possibility}
Schmidhuber, J. 1991.
\newblock A possibility for implementing curiosity and boredom in
  model-building neural controllers.
\newblock In \emph{Proceedings of the International Conference on Simulation of
  Adaptive Behavior: From Animals to Animats}.

\bibitem[{Schulman et~al.(2015)Schulman, Levine, Abbeel, Jordan, and
  Moritz}]{schulman2015trust}
Schulman, J.; Levine, S.; Abbeel, P.; Jordan, M.; and Moritz, P. 2015.
\newblock Trust region policy optimization.
\newblock In \emph{Proceedings of the International Conference on Machine
  Learning}.

\bibitem[{Seo et~al.(2021)Seo, Chen, Shin, Lee, Abbeel, and Lee}]{seo2021state}
Seo, Y.; Chen, L.; Shin, J.; Lee, H.; Abbeel, P.; and Lee, K. 2021.
\newblock State entropy maximization with random encoders for efficient
  exploration.
\newblock In \emph{Proceedings of the International Conference on Machine
  Learning}.

\bibitem[{Shannon(1948)}]{shannon1948entropy}
Shannon, C.~E. 1948.
\newblock A mathematical theory of communication.
\newblock \emph{The Bell system technical journal}.

\bibitem[{Singh et~al.(2003)Singh, Misra, Hnizdo, Fedorowicz, and
  Demchuk}]{singh2003nearest}
Singh, H.; Misra, N.; Hnizdo, V.; Fedorowicz, A.; and Demchuk, E. 2003.
\newblock Nearest neighbor estimates of entropy.
\newblock \emph{American journal of mathematical and management sciences}.

\bibitem[{Sutton and Barto(2018)}]{sutton2018reinforcement}
Sutton, R.~S.; and Barto, A.~G. 2018.
\newblock \emph{Reinforcement learning: An introduction}.
\newblock MIT press.

\bibitem[{Tamar, Glassner, and Mannor(2015)}]{tamar2015optimizing}
Tamar, A.; Glassner, Y.; and Mannor, S. 2015.
\newblock Optimizing the CVaR via sampling.
\newblock In \emph{Proceedings of the AAAI Conference on Artificial
  Intelligence}.

\bibitem[{Tang et~al.(2017)Tang, Houthooft, Foote, Stooke, Chen, Duan,
  Schulman, De~Turck, and Abbeel}]{tang2017exploration}
Tang, H.; Houthooft, R.; Foote, D.; Stooke, A.; Chen, X.; Duan, Y.; Schulman,
  J.; De~Turck, F.; and Abbeel, P. 2017.
\newblock \# exploration: A study of count-based exploration for deep
  reinforcement learning.
\newblock In \emph{Advances in Neural Information Processing Systems}.

\bibitem[{Tarbouriech and Lazaric(2019)}]{tarbouriech2019active}
Tarbouriech, J.; and Lazaric, A. 2019.
\newblock Active exploration in Markov decision processes.
\newblock In \emph{Proceedings of the International Conference on Artificial
  Intelligence and Statistics}.

\bibitem[{Villani(2008)}]{villani2008optimal}
Villani, C. 2008.
\newblock \emph{Optimal transport: old and new}.
\newblock Springer Science \& Business Media.

\bibitem[{Xu et~al.(2018)Xu, Liu, Zhao, and Peng}]{xu2018meta}
Xu, T.; Liu, Q.; Zhao, L.; and Peng, J. 2018.
\newblock Learning to explore via meta-policy gradient.
\newblock In \emph{Proceedings of the International Conference on Machine
  Learning}.

\bibitem[{Yarats et~al.(2021)Yarats, Fergus, Lazaric, and
  Pinto}]{yarats2021reinforcement}
Yarats, D.; Fergus, R.; Lazaric, A.; and Pinto, L. 2021.
\newblock Reinforcement learning with prototypical representations.
\newblock In \emph{Proceedings of the International Conference on Machine
  Learning}.

\bibitem[{Zhang, Cai, and Li(2021)}]{zhang2020exploration}
Zhang, C.; Cai, Y.; and Li, L. H.~J. 2021.
\newblock Exploration by maximizing R{\'e}nyi entropy for reward-free RL
  framework.
\newblock In \emph{Proceedings of the AAAI Conference on Artificial
  Intelligence}.

\bibitem[{Zhang, Ma, and Singla(2020)}]{zhang2020task}
Zhang, X.; Ma, Y.; and Singla, A. 2020.
\newblock Task-agnostic exploration in reinforcement learning.
\newblock In \emph{Advances in Neural Information Processing Systems}.

\bibitem[{Zintgraf et~al.(2019)Zintgraf, Shiarlis, Igl, Schulze, Gal, Hofmann,
  and Whiteson}]{zintgraf2019varibad}
Zintgraf, L.; Shiarlis, K.; Igl, M.; Schulze, S.; Gal, Y.; Hofmann, K.; and
  Whiteson, S. 2019.
\newblock VariBAD: A very good method for bayes-adaptive deep RL via
  meta-learning.
\newblock In \emph{Proceedings of the International Conference on Learning
  Representations}.

\end{thebibliography}

%%%%%%%%%%%%%%%%%%%%%%%%%%%%%%%%%%%%%%%%%%%%%%%%%%%%%%%%%%%%
\clearpage
\onecolumn
\appendix

\section{Relevant Literature}
\label{apx:related_work}

Our work lies at the intersection of unsupervised RL, robust and risk-averse in RL, and meta-RL. We revise below the relevant literature in these three fields. Other problem formulations that have similarities with ours are the \emph{reward-free RL}~\cite{jin2020rewardfree} and \emph{task-agnostic RL}~\cite{zhang2020task}. These frameworks also consider an unsupervised exploration phase (\ie without rewards). Differently to unsupervised RL, the objective of the exploration is to allow the agent to plan for an $\epsilon$-optimal policy for the worst-case reward (reward-free RL) or an oblivious reward (task-agnostic RL), \ie they do not consider a fine-tuning phase. Hence, the unsupervised exploration focuses on the coverage of the collected data rather than the performance of the pre-trained policy. Moreover, the existing works only consider single-environment settings. Another formulation that relates to ours is the \emph{latent MDP} setting~\cite{hallak2015contextual, kwon2021rl}, in which the agent sequentially interact with an MDP drawn from a fixed distribution, similarly to our unsupervised pre-training phase. However, they assume supervised interactions with the class of environments, and they pursue regret minimization rather than policy pre-training.

% Unsupervised RL
\paragraph{Unsupervised RL}
The literature that relates the most with our work is the one pursuing unsupervised RL in a single environment. A complete survey of the existing works along with an empirical comparison of the main methods can be found in~\cite{laskin2021urlb}. Here, we report a brief summary. \citet{laskin2021urlb} group the existing methods in three categories: Knowledge-based, data-based, and competence-based. The knowledge-based methods~\cite{schmidhuber1991possibility, pathak2017curiosity, burda2019rnd, pathak2019self} are those that exploit the unsupervised pre-training to acquire useful knowledge about the environment. The pre-training objective is usually proportional to the prediction error of a learned model. Instead, the data-based methods try to maximize the diversity of the collected samples, for which count-based bonuses~\cite{bellemare2016unifying, ostrovski2017count, tang2017exploration} and MSVE are typical pre-training objectives. Some of the MSVE methods~\citep{hazan2019maxent,lee2019smm} focus on learning a mixture of policies that is collectively MSVE optimal, while other~\citep{tarbouriech2019active, mutti2020intrinsically} casts the MSVE as a dual (or surrogate) linear program in tabular settings. Successive works tackle MSVE at scale with non-parametric entropy estimation~\citep{mutti2020policy,anonymous2021unsupervised, liu2021aps, yarats2021reinforcement, seo2021state}, or introduce variations to the entropy objective, such as geometry-awareness~\citep{guo2021geometric} and R\'enyi generalization~\citep{zhang2020exploration}. Finally, competence-based methods, which include~\cite{gregor2016variational, eysenbach2018diversity, achiam2018variational} and to some extent~\cite{lee2019smm, liu2021aps}, make use of the unsupervised pre-training to acquire a set of useful skills. Their pre-training objective is usually related to the mutual information between a latent skill vector and the state visitation induced by the skill. To the best of our knowledge, all the existing solutions to the unsupervised RL problem are environment-specific (with the exception of~\cite{parisi2021interesting} commented in Section~\ref{sec:related_work}) and do not directly address the multiple-environments setting.

% Robust RL and risk-aversion
\paragraph{Robust RL and Risk-Aversion}
Previous work considered \emph{CVaR optimization} in RL as we do, either to learn a policy that is averse to the risk induced by the volatility of the returns~\citep{tamar2015optimizing,chow2014algorithms} or by changes in the environment dynamics~\citep[\eg][]{rajeswaran2016epopt}. Here we account for a different source of risk, which is the one of running into a particularly unfavorable environment for the pre-trained exploration strategy.

% Meta RL
\paragraph{Meta-RL}
As we mentioned in Section~\ref{sec:related_work}, our framework resembles the \emph{meta-RL} setting \citep{finn2017model}, in which we would call \emph{meta-training} the unsupervised pre-training, and \emph{meta-testing} the supervised fine-tuning. While some methods target exploration in meta-RL \citep[\eg][]{xu2018meta, gupta2018maesn, zintgraf2019varibad}, they usually assume access to rewards during meta-training, with the notable exception of~\citep{gupta2018unsupervised}. To the best of our knowledge, none of the existing works combine reward-free meta-training with a multiple-environments setting.

\section{Preliminary Theoretical Analysis of the Problem}
We aim to theoretically analyze the problem in~\eqref{eq:exp_risk_averse}, and especially, what makes a class of multiple CMPs hard to explore with a unique strategy. This has to be intended as a preliminary discussion on the problem, which could serve as a starting point for future works, rather than a thorough theoretical characterization.
First, it is worth introducing some additional notation.

\paragraph{Lipschitz Continuity}
Let $X,Y$ be two metric sets with metric functions $d_X, d_Y$. We say a function $f:X \to Y$ is $L_f$-Lipschitz continuous if it holds for some constant $L_f$
$$
    d_Y (f(x'), f(x)) \leq L_f d_X (x',x), \forall (x',x) \in X^2,
$$
where the smallest $L_f$ is the Lipschitz constant and the Lipschitz semi-norm is $\norm{f}{L} = \sup_{x',x \in X} \big\lbrace \frac{d_Y (f(x'), f(x))}{d_X (x',x)} : x' \neq x \big\rbrace$.
When dealing with probability distributions we need to introduce a proper metric. Let $p,q$ be two probability measures, we will either consider the Wasserstein metric~\citep{villani2008optimal}, defined as
$$
    \wass (p, q) = \sup_{f} \big\lbrace \big| \int_X f (x) (p(x) - q(x)) \de x \big| : \norm{f}{L} \leq 1 \big\rbrace,
$$
or the Total Variation (TV) metric, defined as
$$
    d_{TV} (p, q) = \frac{1}{2} \int_X \big| p(x) - q(x)  \big| \de x.
$$

Intuitively, learning to explore a class $\vcmp$ with a policy $\pi$ is challenging when the state distributions induced by $\pi$ in different $\cmp \in \vcmp$ are diverse. The more diverse they are, the more their entropy can vary, and the harder is to get a $\pi$ with a large entropy across the class. To measure this diversity, we are interested in the supremum over the distances between the state distributions $(d_{\pi}^{\cmp_1}, \ldots, d_{\pi}^{\cmp_I})$ that a single policy $\pi \in \Pi$ realizes over the class $\vcmp$. We call this measure the \emph{diameter} $\diam_{\vcmp}$ of the class $\vcmp$. Since we have infinitely many policies in $\Pi$, computing $\diam_{\vcmp}$ is particularly arduous. However, we are able to provide an upper bound to $\diam_{\vcmp}$ defined through a Wasserstein metric.
\begin{restatable}[]{ass}{lipschitz_assumption}
    Let $d_{\sspace}$ be a metric on $\sspace$. The class $\vcmp$ is $L_{P^\pi}$-Lipschitz continuous,
    \begin{equation*}
        \wass( P^\pi (\cdot| s'), P^\pi (\cdot| s)) \leq L_{P^\pi} d_{\sspace} (s',s),
        \;\; \forall (s',s) \in \sspace^2,
    \end{equation*}
    where $P^\pi (s|\os) = \int_{\aspace} \pi (\oa|\os) P(s | \os, \oa) \de \oa$ for $P \in \vcmp$, $\pi \in \Pi$, $L_{P^\pi}$ is a constant $L_{P^\pi} < 1$.
    \label{ass:lipschitz}
\end{restatable}
\begin{restatable}[]{thm}{diameter}
    Let $\vcmp$ be a class of CMPs satisfying Ass.~\ref{ass:lipschitz}. Let $d^{\cmp}_{\pi}$ be the marginal state distribution over $T$ steps induced by the policy $\pi$ in $\cmp \in \vcmp$. We can upper bound the diameter $\diam_{\vcmp}$  as
    \begin{align}
        \diam_{\vcmp} := \sup_{\pi \in \Pi,\; \cmp', \cmp \in \vcmp} 
        \wass (d_{\pi}^{\cmp'}, d_{\pi}^{\cmp}) \leq \sup_{P',P \in \vcmp} \frac{1 - L_{P^\pi}^T}{1 - L_{P^\pi}}
        \sup_{s \in \mathcal{S},a \in \mathcal{A}} 
        \wass (P'(\cdot|s,a), P (\cdot|s,a)). \nonumber
    \end{align}
    \label{thm:dimater}
\end{restatable}
Theorem~\ref{thm:dimater} provides a measure to quantify the hardness of the exploration problem in a specific class of CMPs, and to possibly compare one class with another. However, the value of $\diam_{\vcmp}$ might result, due to the supremum over $\Pi$, from a policy that is far away from the policies we actually deploy while learning, say $(\pi_0, \ldots, \pi_{\mathcal{E}}^*)$. To get a finer assessment of the hardness of $\vcmp$ we face in practice, it is worth considering a policy-specific measure to track during the optimization. We call this measure the \emph{$\pi$-diameter} $\diam_{\vcmp} (\pi)$ of the class $\vcmp$. Theorem~\ref{thm:pi_dimater} provides an upper bound to $\diam_{\vcmp} (\pi)$ defined through a convenient TV metric.
\begin{restatable}[]{thm}{piDiameter}
    Let $\vcmp$ be a class of CMPs, let $\pi \in \Pi$ be a policy, and let $d^{\cmp}_{\pi}$ be the marginal state distribution over $T$ steps induced by $\pi$ in $\cmp \in \vcmp$. We can upper bound the $\pi$-diameter $\diam_{\vcmp} (\pi)$ as
    \begin{align}
        \diam_{\vcmp} (\pi) := \sup_{\cmp', \cmp \in \vcmp} 
        d_{TV} (d_{\pi}^{\cmp'}, d_{\pi}^{\cmp}) \leq \sup_{P',P \in \vcmp} T \EV_{\substack{s \sim d_{\pi}^{\cmp} \\ a \sim \pi(\cdot|s)}}
        d_{TV} (P'(\cdot|s,a), P (\cdot|s,a)). \nonumber
    \end{align}
    \label{thm:pi_dimater}
\end{restatable}
The last missing piece we aim to derive is a result to relate the $\pi$-diameter $\diam_{\vcmp} (\pi)$ of the class $\vcmp$ (Theorem~\ref{thm:pi_dimater}) with the actual exploration objective, \ie the entropy of the state visitations induced by the policy $\pi$ over the environments in the class. In the following theorem, we provide an upper bound to the \emph{entropy gap} induced by the policy $\pi$ within the class $\vcmp$.
\begin{restatable}[]{thm}{entropyDiameter}
    Let $\vcmp$ be a class of CMPs, let $\pi \in \Pi$ be a policy and $\diam_{\vcmp} (\pi)$ the corresponding $\pi$-diameter of $\vcmp$. Let $d^{\cmp}_{\pi}$ be the marginal state distribution over $T$ steps induced by $\pi$ in $\cmp \in \vcmp$, and let $\sigma_{\vcmp} \leq \sigma_{\cmp} := \inf_{s \in \sspace} d^{\cmp}_{\pi} (s), \forall \cmp \in \vcmp$. We can upper bound the entropy gap of the policy $\pi$ within the model class $\vcmp$ as
    \begin{align}
        \sup_{\cmp',\cmp \in \vcmp} \big|  H (d_{\pi}^{\cmp'}) - H (d_{\pi}^{\cmp}) \big| \leq \big( \diam_{\vcmp} (\pi) \big)^2 \big/ \sigma_{\vcmp} + \diam_{\vcmp} (\pi) \log  (1  / \sigma_{\vcmp}) \nonumber
    \end{align}
    \label{thm:entropy_dimater}
\end{restatable}

\label{sec:theory}

\section{Proofs}
\label{apx:proofs}
\cvarPolicyGradient*

\begin{proof}
Let us start from expanding the exploration objective~\eqref{eq:exp_risk_averse} to write
\begin{align}
    \mathcal{E}_{\vcmp}^\alpha (\pi) \nonumber &= \cvar (H_\tau)  \\
    &= \EV_{\substack{\cmp \sim p_{\vcmp} \\ \tau \sim p_{\pi, \cmp}}} \big[ H_{\tau} \;|\; H_{\tau} \leq \var (H_\tau) \big] 
    = \frac{1}{\alpha} \int_{- \infty}^{\var(H_\tau)} p_{\pi_{\vtheta}, \vcmp} (h) h \de h,
    \label{eq:cvar_integral}
\end{align}
where $p_{\pi_{\vtheta}, \vcmp}$ is the probability density function (pdf) of the random variable $H_\tau$ when the policy $\pi_{\vtheta}$ is deployed on the class of environments $\vcmp$, and the last equality comes from the definition of CVaR~\citep{rockafellar2000optimization}.
Before computing the gradient of~\eqref{eq:cvar_integral}, we derive a preliminary result for later use, \ie
\begin{align}
    \nabla_{\vtheta} &\int_{- \infty}^{\var(H_\tau)}  p_{\pi_{\vtheta}, \vcmp} (h) \de h \nonumber \\
    &= \int_{- \infty}^{\var(H_\tau)} \nabla_{\vtheta} p_{\pi_{\vtheta}, \vcmp} (h) \de h
    + \nabla_{\vtheta} \var(H_\tau) p_{\pi_{\vtheta}, \vcmp} (\var(H_\tau)) = 0,
    \label{eq:cvar_grad_pre}
\end{align}
which follows directly from the Leibniz integral rule, noting that $\var(H_\tau)$ depends on $\vtheta$ through the pdf of $H_\tau$. We now take the gradient of~\eqref{eq:cvar_integral} to get
\begin{align}
    &\nabla_{\vtheta} \mathcal{E}_{\vcmp}^\alpha (\pi) \nonumber \\
    &= \nabla_{\vtheta} \frac{1}{\alpha} \int_{- \infty}^{\var(H_\tau)} p_{\pi_{\vtheta}, \vcmp} (h) h \de h \nonumber \\
    &= \frac{1}{\alpha} \int_{- \infty}^{\var(H_\tau)} \nabla_{\vtheta} p_{\pi_{\vtheta}, \vcmp} (h) h \de h
    + \frac{1}{\alpha} \nabla_{\vtheta} \var(H_\tau) \var(H_\tau)  p_{\pi_{\vtheta}, \vcmp} (\var(H_\tau)) \label{eq:cvar_grad_1} \\
    &= \frac{1}{\alpha} \int_{- \infty}^{\var(H_\tau)} \nabla_{\vtheta} p_{\pi_{\vtheta}, \vcmp} (h) \bigg( h - \var (H_\tau) \bigg) \de h, \label{eq:cvar_grad_2} 
\end{align}
where \eqref{eq:cvar_grad_1} follows from the Leibniz integral rule, and \eqref{eq:cvar_grad_2} is obtained from \eqref{eq:cvar_grad_1} through \eqref{eq:cvar_grad_pre}, which we can rearrange to write $ p_{\pi_{\vtheta}, \vcmp} (\var(H_\tau)) = \frac{1}{\nabla_{\vtheta} \var(H_\tau)} \int_{- \infty}^{\var(H_\tau)} \nabla_{\vtheta} p_{\pi_{\vtheta}, \vcmp} (h) \de h$.
All of the steps above are straightforward replications of the derivations by Tamar et al.~\cite{tamar2015optimizing}, Proposition 1. To conclude the proof we just have to compute the term $\nabla_{\vtheta} p_{\pi_{\vtheta}, \vcmp} (h)$, which is specific to our setting. Especially, we note that
\begin{align}
    \nabla_{\vtheta} &p_{\pi_{\vtheta}, \vcmp} (h)  \nonumber \\
    &= \int_{\vcmp} p_{\vcmp} (\cmp) \int_{\Tau} \nabla_{\vtheta} p_{\pi_{\vtheta},\cmp} (\tau) \delta(h - H_\tau) \de \tau \de \cmp \label{eq:cvar_grad_3} \\
    &= \int_{\vcmp} p_{\vcmp} (\cmp) \int_{\Tau} p_{\pi_{\vtheta},\cmp} (\tau) \nabla_{\vtheta} \log p_{\pi_{\vtheta},\cmp} (\tau) \delta(h - H_\tau) \de \tau \de \cmp \nonumber \\
    &= \int_{\vcmp} p_{\vcmp} (\cmp) \int_{\Tau} p_{\pi_{\vtheta},\cmp} (\tau) \bigg( \sum_{t = 0}^{T - 1} \nabla_{\vtheta} \log \pi_{\vtheta} (a_{t,\tau}|s_{t,\tau}) \bigg) \delta(h - H_\tau) \de \tau \de \cmp, \label{eq:cvar_grad_4}
\end{align}
where \eqref{eq:cvar_grad_3} and \eqref{eq:cvar_grad_4} are straightforward from the definitions in Section~\ref{sec:preliminaries}, and $\Tau$ is the set of feasible trajectories of length $T$. Finally, the result follows by plugging  \eqref{eq:cvar_grad_4} into \eqref{eq:cvar_grad_2}, which gives
\begin{align*}
    \nabla_{\vtheta} &\mathcal{E}_{\vcmp}^\alpha (\pi) 
    = \frac{1}{\alpha}
    \int_{\vcmp} p_{\vcmp} (\cmp) \int_{\Tau} p_{\pi_{\vtheta},\cmp} (\tau) \\
    &\quad\times \int_{- \infty}^{\var(H_\tau)} \delta(h - H_\tau) \bigg( \sum_{t = 0}^{T - 1} \nabla_{\vtheta} \log \pi_{\vtheta} (a_{t,\tau}|s_{t,\tau}) \bigg)  \bigg( h - \var (H_\tau) \bigg) \de h \de \tau \de \cmp.
\end{align*}
\end{proof}

\diameter*

\begin{proof}
The proof follows techniques from~\cite{pirotta2015policy}. Let us report a preliminary result which states that the function $h_f (\os) = \int_{\aspace} \pi(\oa|\os) \int_{\sspace} P(s | \os, \oa) \de s \de \oa $ has a Lipschitz constant equal to $L_{P^\pi}$ \citep[][Lemma 3]{pirotta2015policy}:
\begin{align}
    \big| h_f (\os') - h_f (\os) \big| 
    &= \bigg| \int_{\sspace} f(s) \int_{\aspace} \pi (a | \os') P (s | \os', a) \de a \de s
    -  \int_{\sspace} f(s) \int_{\aspace} \pi (a | \os) P (s | \os, a) \de a \de s \bigg| \nonumber \\
    &= \bigg| \int_{\sspace} f(s) \bigg( P^{\pi} (s | \os')- P^{\pi} (s | \os) \bigg) \de s \bigg|
    \leq L_{P^\pi} d_{\sspace} ( \os', \os ), \label{eq:h_function}
\end{align}
where $d_{\sspace}$ is a metric over $\sspace$ and $P^\pi (s|\os) = \int_{\aspace} \pi (\oa|\os) P(s | \os, \oa) \de \oa$.
Then, we note that the marginal state distribution over $T$ steps $d_{\pi}^{\cmp}$ can be written as a sum of the contributions $d_{\pi,t}^{\cmp}$ related to any time step $t \in [T]$, which is
\begin{equation}
    d_{\pi}^{\cmp} (s) = \frac{1}{T} \sum_{t = 0}^{T - 1} d_{\pi,t}^{\cmp} (s). \label{eq:marginal_state_distribution}
\end{equation}
Hence, we can look at the Wasserstein distance of the state distributions for some $t \in [T]$ and $\cmp',\cmp \in \vcmp$. We obtain
\begin{align}
    \wass &( d_{\pi,t}^{\cmp'}, d_{\pi,t}^{\cmp} ) \nonumber \\
    &=\sup_f \bigg\lbrace 
    \bigg| \int_{\sspace} \bigg( d_{\pi,t}^{\cmp'} (s) - d_{\pi,t}^{\cmp} (s) \bigg) f(s) \de s \bigg| : \norm{f}{L} \leq 1
    \bigg\rbrace \label{eq:wass_t_1} \\
    &= \sup_f \bigg\lbrace 
    \bigg| \int_{\sspace} \int_{\aspace}
    \int_{\sspace}  \bigg( d_{\pi,t-1}^{\cmp'} (\os) \pi(\oa|\os) P'(s|\os,\oa) 
    - d_{\pi,t-1}^{\cmp} (\os) \pi(\oa|\os) P(s|\os,\oa) \bigg) f(s) \de s \de \oa \de \os 
    \bigg| : \norm{f}{L} \leq 1
    \bigg\rbrace \nonumber \\
    &= \sup_f \bigg\lbrace 
    \bigg| \int_{\sspace} d_{\pi,t-1}^{\cmp'} (\os) \int_{\aspace}
    \int_{\sspace} \pi(\oa|\os) \bigg( P'(s|\os,\oa) - P(s|\os,\oa) \bigg) 
    f(s) \de s \de \oa \de \os \label{eq:wass_t_2} \\
    &\quad+ \int_{\sspace} \bigg( d_{\pi,t-1}^{\cmp'} (\os) - d_{\pi,t-1}^{\cmp} (\os) \bigg) \int_{\aspace} \int_{\sspace} \pi(\oa|\os) P(s|\os,\oa)
    f(s) \de s \de \oa \de \os
    \bigg| : \norm{f}{L} \leq 1
    \bigg\rbrace \label{eq:wass_t_3} \\
    &\leq \sup_f \bigg\lbrace 
    \bigg| \int_{\sspace} d_{\pi,t-1}^{\cmp'} (\os) \int_{\aspace}
    \int_{\sspace} \pi(\oa|\os) \bigg( P'(s|\os,\oa) - P(s|\os,\oa) \bigg) 
    f(s) \de s \de \oa \de \os
    \bigg| : \norm{f}{L} \leq 1
    \bigg\rbrace \nonumber \\
    &\quad+ \sup_f \bigg\lbrace \bigg| 
    \int_{\sspace} \bigg( d_{\pi,t-1}^{\cmp'} (\os) - d_{\pi,t-1}^{\cmp} (\os) \bigg) \int_{\aspace} \int_{\sspace} \pi(\oa|\os) P(s|\os,\oa)
    f(s) \de s \de \oa \de \os
    \bigg| : \norm{f}{L} \leq 1
    \bigg\rbrace \nonumber \\
    &\leq \sup_f \bigg\lbrace 
    \int_{\sspace} d_{\pi,t-1}^{\cmp'} (\os) \int_{\aspace}
    \pi(\oa|\os) \de \oa \de \os
    \sup_{\os \in \sspace, \oa \in \aspace} \bigg\lbrace \bigg| \int_{\sspace}\bigg( P'(s|\os,\oa) - P(s|\os,\oa) \bigg)
    f(s) \de s 
    \bigg| \bigg\rbrace : \norm{f}{L} \leq 1
    \bigg\rbrace \nonumber \\
    &\quad+ L_{P^\pi} \sup_f \bigg\lbrace \bigg| 
    \int_{\sspace} \bigg( d_{\pi,t-1}^{\cmp'} (\os) - d_{\pi,t-1}^{\cmp} (\os) \bigg) \frac{h_f (\os)}{L_{P^\pi}} \de \os
    \bigg| : \norm{f}{L} \leq 1
    \bigg\rbrace \label{eq:wass_t_4} \\
    &= \sup_{s \in \sspace, a \in \aspace} \wass ( P' (\cdot| s, a), P (\cdot| s, a) )
    + L_{P^\pi} \wass ( d_{\pi,t-1}^{\cmp'}, d_{\pi,t-1}^{\cmp} ), \label{eq:wass_t_5}
\end{align}
where we plugged the common temporal relation $d_{\pi,t}^{\cmp} (s') = \int_{\sspace}\int_{\aspace} d_{\pi,t-1}^{\cmp} (s) \pi(a | s) P(s'|s,a) \de s \de a$ into \eqref{eq:wass_t_1}, we sum and subtract $\int_{\sspace} \int_{\aspace} \int_{\sspace} d_{\pi,t-1}^{\cmp'} (\os) \pi(\oa|\os) P(s|\os,\oa) \de s \de \oa \de \os$ to get \eqref{eq:wass_t_2}, \eqref{eq:wass_t_3}, and we apply the inequality in \eqref{eq:h_function} to obtain \eqref{eq:wass_t_4} and then \eqref{eq:wass_t_5}. To get rid of the dependence to the state distributions $d_{\pi,t-1}^{\cmp'}$ and $d_{\pi,t-1}^{\cmp}$, we repeatedly unroll \eqref{eq:wass_t_5} to get
\begin{align}
    \wass ( d_{\pi,t}^{\cmp'}, d_{\pi,t}^{\cmp} ) 
    &\leq \bigg( \sum_{j = 0}^t L_{P^\pi}^j \bigg) \sup_{s \in \sspace, a \in \aspace} \wass ( P' (\cdot| s, a), P (\cdot| s, a))
    + L_{P^\pi}^t  \wass ( D', D ) \label{eq:unrolling_1} \\
    &= \bigg( \frac{1 - L_{P^\pi}^t}{1 - L_{P^\pi}} \bigg) \sup_{s \in \sspace, a \in \aspace} \wass ( P' (\cdot| s, a), P (\cdot| s, a))
    + L_{P^\pi}^t  \wass ( D', D ), \label{eq:unrolling_2}
\end{align}
where we note that $\wass ( d_{\pi,0}^{\cmp'}, d_{\pi,0}^{\cmp} ) = \wass ( D', D )$ to derive \eqref{eq:unrolling_1}, and we assume $L_{P^\pi} < 1$ (Assumption~\ref{ass:lipschitz}) to get \eqref{eq:unrolling_2} from \eqref{eq:unrolling_1}. As a side note, when the state and action spaces are discrete, a natural choice of a metric is $d_{\sspace}(s',s) = \ind (s' \neq s )$ and $d_{\aspace} = \ind (a' \neq a)$, which results in the Wasserstein distance being equivalent to the total variation, the constant $L_{P^\pi} = 1$, and $\sum_{j = 0}^t L_{P^\pi}^j = t$. More details over the Lipschitz constant $L_{P^\pi}$ can be found in \cite{pirotta2015policy}.
Finally, we can exploit the result in \eqref{eq:unrolling_2} to write
\begin{align}
    \wass ( d_{\pi}^{\cmp'}, d_{\pi}^{\cmp} ) &= 
    \sup_{f} \bigg\lbrace \bigg| 
    \int_{\sspace} \bigg( \frac{1}{T} \sum_{t=0}^{T-1} d_{\pi,t}^{\cmp'} (s)
    - \frac{1}{T} \sum_{t=0}^{T-1}d_{\pi,t}^{\cmp} (s) \bigg) f(s) \de s
    \bigg| : \norm{f}{L} \leq 1 \bigg\rbrace \label{eq:wass_marg_1} \\
    &\leq \frac{1}{T} \sum_{t=0}^{T-1} \sup_{f} \bigg\lbrace \bigg| 
    \int_{\sspace} \bigg(d_{\pi,t}^{\cmp'} (s)
    - d_{\pi,t}^{\cmp} (s) \bigg) f(s) \de s
    \bigg| : \norm{f}{L} \leq 1 \bigg\rbrace \nonumber \\
    &\leq \frac{1}{T} \sum_{t=0}^{T-1} \frac{1 - L_{P^\pi}^t}{1 - L_{P^\pi}} \sup_{s \in \sspace, a \in \aspace} \wass ( P' (\cdot| s, a), P (\cdot| s, a) ) + L_{P^\pi}^t  \wass ( D', D )
    \nonumber \\
    &\leq \frac{1 - L_{P^\pi}^T}{1 - L_{P^\pi}} \sup_{s \in \sspace, a \in \aspace} \wass ( P' (\cdot| s, a), P (\cdot| s, a) )
    + L_{P^\pi}^T \wass ( D', D ), \label{eq:wass_marg_2}
\end{align}
in which we use \eqref{eq:marginal_state_distribution} to get \eqref{eq:wass_marg_1}. The result follows from \eqref{eq:wass_marg_2} by assuming the initial state distribution $D$ to be shared across all the CMPs in $\vcmp$, and taking the supremum over $P', P \in \vcmp$.
\end{proof}

\piDiameter*

\begin{proof}
The proof follows techniques from~\cite{metelli2018configurable}, especially Proposition 3.1. Without loss of generality, we take $\cmp', \cmp \in \vcmp$. With some overloading of notation, we will alternatively identify a CMP with the tuple $\cmp$ or its transition model $P$. Let us start considering the TV between the marginal state distributions induced by $\pi$ over $\cmp',\cmp$, we can write
\begin{align}
    d_{TV} &(d_{\pi}^{\cmp'},d_{\pi}^{\cmp})  \nonumber \\
    &= \frac{1}{2} \int_{\sspace} \big| d_{\pi}^{\cmp'} (s) - d_{\pi}^{\cmp} (s) \big| \de s
    = \frac{1}{2} \int_{\sspace} \bigg| \frac{1}{T} \sum_{t = 0}^{T - 1} d_{\pi,t}^{\cmp'} (s) - \frac{1}{T} \sum_{t = 0}^{T - 1} d_{\pi,t}^{\cmp} (s) \bigg| \de s \label{eq:tv_mar_1} \\
    &\leq \frac{1}{2T} \sum_{t = 0}^{T - 1} \int_{\sspace} \big| d_{\pi,t}^{\cmp'} (s) -  d_{\pi,t}^{\cmp} (s) \big| \de s
    = \frac{1}{T} \sum_{t = 0}^{T - 1} d_{TV} (d_{\pi,t}^{\cmp'},d_{\pi,t}^{\cmp}), \label{eq:tv_mar_2}
\end{align}
where we use \eqref{eq:marginal_state_distribution} to get \eqref{eq:tv_mar_1}. Then, we provide an upper bound to each term of the final sum in \eqref{eq:tv_mar_2}, \ie
\begin{align}
    d_{TV} &(d_{\pi,t}^{\cmp'},d_{\pi,t}^{\cmp}) \nonumber \\
    &= \frac{1}{2} \int_{\sspace} \big| d_{\pi,t}^{\cmp'}(s) - d_{\pi,t}^{\cmp}(s) \big| \de s \nonumber \\
    &= \frac{1}{2} \int_{\sspace} \bigg| 
    \int_{\aspace} \int_{\sspace}  d_{\pi,t-1}^{\cmp'}(\os) \pi(\oa|\os) P'(s|\os,\oa)
    - d_{\pi,t-1}^{\cmp}(\os) \pi(\oa|\os) P(s|\os,\oa) \bigg| \de \os \de \oa \de s 
    \label{eq:tv_t_1}\\
    &\leq \frac{1}{2} 
    \int_{\sspace} \big| d_{\pi,t-1}^{\cmp'}(\os) - d_{\pi,t-1}^{\cmp}(\os)  \big| \int_{\aspace} \int_{\sspace} \pi(\oa|\os) P'(s|\os,\oa) \de \os \de \oa \de s \label{eq:tv_t_2} \\
    &\quad+ \frac{1}{2} 
    \int_{\sspace} \int_{\aspace}  d_{\pi,t-1}^{\cmp}(\os) \pi(\oa|\os) \int_{\sspace} \big| P'(s|\os,\oa) - P(s|\os,\oa) \big| \de \os \de \oa \de s \label{eq:tv_t_3}\\
    &= d_{TV} (d_{\pi,t-1}^{\cmp'},d_{\pi,t-1}^{\cmp})
    + \EV_{\substack{s \sim d_{\pi,t-1}^{\cmp} \\ a \sim \pi (\cdot|s) }}
    \bigg[ d_{TV} (P'(\cdot|s,a),P(\cdot|s,a)) \bigg] \label{eq:tv_t_4} \\
    &= \sum_{j = 1}^{t - 1} \EV_{\substack{s \sim d_{\pi,j}^{\cmp} \\ a \sim \pi (\cdot|s) }}
    \bigg[ d_{TV} (P'(\cdot|s,a),P(\cdot|s,a)) \bigg] + d_{TV} (D', D), \label{eq:tv_t_5}
\end{align}
where we use the temporal relation $d_{\pi,t}^{\cmp} (s') = \int_{\sspace}\int_{\aspace} d_{\pi,t-1}^{\cmp} (s) \pi(a | s) P(s'|s,a) \de s \de a$ to get \eqref{eq:tv_t_1}, in which we sum and subtract $\int_{\sspace} \int_{\aspace} \int_{\sspace} d_{\pi,t-1}^{\cmp} (\os) \pi(\oa|\os) P(s|\os,\oa) \de s \de \oa \de \os$ to obtain \eqref{eq:tv_t_2} and \eqref{eq:tv_t_3}, and we repeatedly unroll \eqref{eq:tv_t_4} to write \eqref{eq:tv_t_5}, noting that $d_{TV} (d_{\pi,0}^{\cmp'},d_{\pi,0}^{\cmp}) = d_{TV} (D', D)$.
Finally, we can plug  \eqref{eq:tv_t_5} in \eqref{eq:tv_mar_2} to get
\begin{align}
    d_{TV} &(d_{\pi}^{\cmp'},d_{\pi}^{\cmp}) \nonumber \\
    &\leq \frac{1}{T} \sum_{t = 0}^{T - 1} d_{TV} (d_{\pi,t}^{\cmp'},d_{\pi,t}^{\cmp}) \nonumber \\
    &\leq \frac{1}{T} \sum_{t = 0}^{T - 1} \sum_{j = 1}^{t - 1} \EV_{\substack{s \sim d_{\pi,j}^{\cmp} \\ a \sim \pi (\cdot|s) }}
    \bigg[ d_{TV} (P'(\cdot|s,a),P(\cdot|s,a)) \bigg] + d_{TV} (D', D) \nonumber \\
    &\leq  \sum_{t = 0}^{T - 1} \int_{\sspace} \frac{1}{T} \sum_{j = 0}^{T - 1} d_{\pi,j}^{\cmp} (s) \EV_{\substack{a \sim \pi (\cdot|s) }}
    \bigg[ d_{TV} (P'(\cdot|s,a),P(\cdot|s,a)) \bigg] \de s + d_{TV} (D', D) \label{eq:tv_mar_3} \\
    &= \sum_{t = 0}^{T - 1} \EV_{\substack{s \sim d_{\pi}^{\cmp} \\ a \sim \pi (\cdot|s) }}
    \bigg[ d_{TV} (P'(\cdot|s,a),P(\cdot|s,a)) \bigg] + d_{TV} (D', D) \label{eq:tv_mar_4}\\
    &= T \EV_{\substack{s \sim d_{\pi}^{\cmp} \\ a \sim \pi (\cdot|s) }}
    \bigg[ d_{TV} (P'(\cdot|s,a),P(\cdot|s,a)) \bigg] + d_{TV} (D', D), \label{eq:tv_mar_5}
\end{align}
in which we have used \eqref{eq:marginal_state_distribution} to obtain \eqref{eq:tv_mar_4} from \eqref{eq:tv_mar_3}. The final result is straightforward from \eqref{eq:tv_mar_4} by assuming the initial state distribution $D$ to be shared across all the CMPs in $\vcmp$, and taking the supremum over $P', P \in \vcmp$.
\end{proof}

\entropyDiameter*

\begin{proof}
Let us expand the entropy gap of the policy $\pi$ as
\begin{align}
    \sup_{\cmp',\cmp \in \vcmp} & \big| H (d_{\pi}^{\cmp'}) - H (d_{\pi}^{\cmp}) \big| \nonumber \\
    &=\sup_{\cmp',\cmp \in \vcmp} \bigg\{ \bigg| - \int_{\sspace} d_{\pi}^{\cmp'} (s) \log d_{\pi}^{\cmp'} (s) \de s + \int_{\sspace} d_{\pi}^{\cmp} (s) \log d_{\pi}^{\cmp} (s) \de s \bigg| \bigg\} \label{eq:ent_gap_0} \\
    &\leq \sup_{\cmp',\cmp \in \vcmp} \bigg\{ \bigg| \int_{\sspace} \Big( d_{\pi}^{\cmp} (s) - d_{\pi}^{\cmp'} (s) \Big) \log d_{\pi}^{\cmp} (s) \de s  \bigg| 
    +  \bigg| \int_{\sspace} d_{\pi}^{\cmp'} (s) \Big( \log d_{\pi}^{\cmp'} (s) - \log  d_{\pi}^{\cmp} (s) \Big) \de s  \bigg| \bigg\} \label{eq:ent_gap_1} \\
    &\leq \sup_{\cmp',\cmp \in \vcmp} \bigg\{ - \log \sigma_{\cmp} \int_{\sspace} \Big| d_{\pi}^{\cmp'} (s) - d_{\pi}^{\cmp'} (s) \Big| \de s 
    +  D_{KL} \big( d_{\pi}^{\cmp'} || d_{\pi}^{\cmp} \big) \bigg\} \label{eq:ent_gap_2} \\
    &\leq \sup_{\cmp',\cmp \in \vcmp} \bigg\{ - \log \sigma_{\cmp} D_{TV} ( d_{\pi}^{\cmp'}, d_{\pi}^{\cmp} )
    +  \big( D_{TV} ( d_{\pi}^{\cmp'}, d_{\pi}^{\cmp} ) \big)^2 \big/ \sigma_{\cmp} \bigg\} \label{eq:ent_gap_3} \\
    &\leq \big( \diam_{\vcmp} (\pi) \big)^2 \big/ \sigma_{\vcmp} - \diam_{\vcmp} (\pi) \log \sigma_{\vcmp}
\end{align}
in which we sum and subtract $\int_{\sspace} d_{\pi}^{\cmp'} (s) \log d_{\pi}^{\cmp} (s) \de s$ to obtain \eqref{eq:ent_gap_1} from \eqref{eq:ent_gap_0}, $\log d_{\pi}^{\cmp} (s)$ is upper bounded with $\log \sigma_{\cmp}$ to get $\eqref{eq:ent_gap_2}$, and we use the reverse Pinsker's inequality $D_{KL} (p || q) \leq (D_{TV} (p, q) )^2 / \inf_{ x \in \mathcal{X}} q (x)$ \citep[][p. 1012 and Lemma 6.3]{csiszar2006context} to obtain \eqref{eq:cvar_grad_3}. Finally, we get the result by upper bounding $D_{TV} (d_{\pi}^{\cmp'}, d_{\pi}^{\cmp})$ with the $\pi$-diameter $\diam_{\vcmp} (\pi)$ and $\sigma_{\cmp}$ with $\sigma_{\vcmp}$ in \eqref{eq:ent_gap_3}.
\end{proof}

\section{Algorithm}
\label{apx:algorithm}
In this section, we provide additional details about $\alpha$MEPOL. A full implementation of the algorithm can be found at \url{https://github.com/muttimirco/alphamepol}.

\subsection{The Benefits of the Baseline}
\label{apx:algorithm_baseline}
In this section, we provide theoretical and empirical motivations to corroborate the use of the baseline $b = - \var (H_\tau)$ into the Monte Carlo policy gradient estimator (Section~\ref{sec:method}, Equation~\ref{eq:policy_gradient_estimator}). Thus, we compare the properties of two alternatives policy gradient estimator, with and without a baseline, \ie
\begin{gather}
    \widehat{\nabla}_{\vtheta} \mathcal{E}_{\vcmp}^\alpha (\pi_{\vtheta})
    = \frac{1}{\alpha N} \sum_{i = 1}^N  f_{\tau_i} \ \big( \widehat{H}_{\tau_i} - \evar (H_{\tau_i}) \big) \ \ind (\widehat{H}_{\tau_i} \leq \evar (H_\tau)), \nonumber \\
    \widehat{\nabla}^b_{\vtheta} \mathcal{E}_{\vcmp}^\alpha (\pi_{\vtheta})
    = \frac{1}{\alpha N} \sum_{i = 1}^N  f_{\tau_i} \ \big( \widehat{H}_{\tau_i} - \var (H_{\tau_i}) - b \big) \ \ind (\widehat{H}_{\tau_i} \leq \evar (H_\tau)). \nonumber
\end{gather}
where $f_{\tau_i} = \sum_{t = 0}^{T - 1} \nabla_{\vtheta} \log \pi_{\vtheta} (a_{t,\tau_i}|s_{t,\tau_i})$.
The former ($\widehat{\nabla}_{\vtheta} \mathcal{E}_{\vcmp}^\alpha$) is known to be asymptotically unbiased~\citep{tamar2015optimizing}, but it is hampered by the estimation error of the VaR term to be subtracted to each $\widehat{H}_{\tau_i}$ in finite sample regimes~\citep{kolla2019concentration}. The latter ($\widehat{\nabla}^b_{\vtheta} \mathcal{E}_{\vcmp}^\alpha$) introduces some bias in the estimate, but it crucially avoids the estimation error of the VaR term to be subtracted, as it cancels out with the baseline $b$. The following proposition, along with related lemmas, assesses the critical number of samples ($n^*$) for which an upper bound to the bias of $\widehat{\nabla}^b_{\vtheta} \mathcal{E}_{\vcmp}^\alpha$ is lower to the estimation error of $\widehat{\nabla}_{\vtheta} \mathcal{E}_{\vcmp}^\alpha$.

\begin{restatable}[]{lemma}{biasUpperBound}
The expected bias of the policy gradient estimate $\widehat{\nabla}^b_{\vtheta} \mathcal{E}_{\vcmp}^\alpha (\pi_{\vtheta})$ can be upper bounded as
\begin{equation*}
    \EV_{\substack{ \cmp \sim \vcmp \\ \tau_i \sim p_{\pi_{\vtheta}, \cmp} }} \big[ \text{\emph{bias}}  \big] 
    = \EV_{\substack{ \cmp_i \sim \vcmp \\ \tau_i \sim p_{\pi_{\vtheta}, \cmp_i} }} 
    \big[ \nabla_{\vtheta} \mathcal{E}_{\vcmp}^\alpha (\pi_{\vtheta}) -\widehat{\nabla}^b_{\vtheta} \mathcal{E}_{\vcmp}^\alpha (\pi_{\vtheta})  \big]
    \leq \mathcal{U} \alpha  b, 
\end{equation*}
where $\mathcal{U}$ is a constant such that $f_{\tau_i} \leq \mathcal{U}$ for all $\tau_i$.
\label{lemma:bias_upper_bound}
\end{restatable}
\begin{proof}
This Lemma can be easily derived by means of
\begin{align}
    &\EV_{\substack{ \cmp_i \sim \vcmp \\ \tau_i \sim p_{\pi_{\vtheta}, \cmp_i} }} \big[ \text{bias}  \big] \nonumber \\
    &= \EV_{\substack{ \cmp_i \sim \vcmp \\ \tau_i \sim p_{\pi_{\vtheta}, \cmp_i} }} 
    \bigg[ \nabla_{\vtheta} \mathcal{E}_{\vcmp}^\alpha (\pi_{\vtheta}) -\widehat{\nabla}^b_{\vtheta} \mathcal{E}_{\vcmp}^\alpha (\pi_{\vtheta})  \bigg] \nonumber \\
    &= \nabla_{\vtheta} \mathcal{E}_{\vcmp}^\alpha (\pi_{\vtheta}) - \EV_{\substack{ \cmp_i \sim \vcmp \\ \tau_i \sim p_{\pi_{\vtheta}, \cmp_i} }} 
    \bigg[ \frac{1}{\alpha N} \sum_{i = 1}^N  f_{\tau_i} \ \big( \widehat{H}_{\tau_i} - \var (H_{\tau_i}) - b \big) \ \ind (\widehat{H}_{\tau_i} \leq \evar (H_\tau))  \bigg] \nonumber \\
    &= \nabla_{\vtheta} \mathcal{E}_{\vcmp}^\alpha (\pi_{\vtheta}) - \EV_{\substack{ \cmp \sim \vcmp \\ \tau \sim p_{\pi_{\vtheta}, \cmp} }} 
    \bigg[   f_{\tau} \ \big( \widehat{H}_{\tau} - \var (H_{\tau}) - b \big) \ \ind (\widehat{H}_{\tau} \leq \evar (H_\tau))  \bigg] \label{eq:lemma_bias_1} \\
    &= \nabla_{\vtheta} \mathcal{E}_{\vcmp}^\alpha (\pi_{\vtheta}) - \nabla_{\vtheta} \mathcal{E}_{\vcmp}^\alpha (\pi_{\vtheta}) +
    \EV_{\substack{ \cmp \sim \vcmp \\ \tau \sim p_{\pi_{\vtheta}, \cmp} }} 
    \bigg[   f_{\tau} \ b \ \ind (\widehat{H}_{\tau} \leq \evar (H_\tau))  \bigg] \label{eq:lemma_bias_2}  \\
    &= \EV_{\substack{ \cmp \sim \vcmp \\ \tau \sim p_{\pi_{\vtheta}, \cmp} }} 
    \bigg[   f_{\tau} \  b \ \ind (\widehat{H}_{\tau} \leq \evar (H_\tau))  \bigg] \leq \mathcal{U} \alpha b, \label{eq:lemma_bias_3}
\end{align}
where \eqref{eq:lemma_bias_2} follows from \eqref{eq:lemma_bias_1} by noting that the estimator without the baseline term is unbiased~\citep{tamar2015optimizing}, and \eqref{eq:lemma_bias_3} is obtained by upper bounding $f_\tau$ with $\mathcal{U}$ and noting that $\EV_{\substack{ \cmp \sim \vcmp \\ \tau \sim p_{\pi_{\vtheta}, \cmp} }} \big[ \ind (\widehat{H}_{\tau} \leq \evar (H_\tau)) \big] = \alpha$.
\end{proof}

\begin{restatable}[VaR concentration bound from~\cite{kolla2020concentration}]{lemma}{varConcentration}
Let $X$ be a continuous random variable with a pdf $f_X$ for which there exist $\eta, \Delta > 0$ such that $f_X (x) > \eta$ for all $x \in \big[ \var (X) - \frac{\Delta}{2}, \var(X) + \frac{\Delta}{2} \big]$.
Then, for any $\epsilon > 0$ we have
\begin{equation*}
    Pr \big[| \evar (X)_\alpha - \var (X) | \geq \epsilon \big] 
    \leq 2 \exp \big( - 2 n \eta^2 \min( \epsilon^2, \Delta^2 ) \big),
\end{equation*}
where $n \in \mathbb{N}$ is the number of samples employed to estimate $\evar (X)$.
\label{lemma:var_concentration}
\end{restatable}

\begin{restatable}[]{prop}{numberOfSamples}
Let $\widehat{\nabla}_{\vtheta} \mathcal{E}_{\vcmp}^\alpha (\pi_{\vtheta})$ and $\widehat{\nabla}^b_{\vtheta} \mathcal{E}_{\vcmp}^\alpha (\pi_{\vtheta})$ be policy gradient estimates with and without a baseline. 
Let $f_{H}$ be the pdf of $H_\tau$, for which there exist $\eta, \Delta > 0$ such that $f_H (H_\tau) > \eta$ for all $H_\tau \in \big[ \var (H_\tau) - \frac{\Delta}{2}, \var (H_\tau) + \frac{\Delta}{2} \big]$.
The number of samples $n^*$ for which the estimation error $\epsilon$ of $\widehat{\nabla}_{\vtheta} \mathcal{E}_{\vcmp}^\alpha (\pi_{\vtheta})$ is lower than the bias of $\widehat{\nabla}^b_{\vtheta} \mathcal{E}_{\vcmp}^\alpha (\pi_{\vtheta})$ with at least probability $\delta \in (0, 1)$ is given by
\begin{equation*}
    n^* = \frac{\log 2 / \delta}{2 \eta^2 \min(\mathcal{U}^2 \alpha^2 b^2,\Delta^2)}.
\end{equation*}
\label{prop:number_of_samples}
\end{restatable}
\begin{proof}
The proof is straightforward by considering the estimation error $\epsilon$ of $\widehat{\nabla}_{\vtheta} \mathcal{E}_{\vcmp}^\alpha (\pi_{\vtheta})$ equal to the upper bound of the bias of $\widehat{\nabla}^b_{\vtheta} \mathcal{E}_{\vcmp}^\alpha (\pi_{\vtheta})$ from Lemma~\ref{lemma:bias_upper_bound}, \ie $\epsilon = \mathcal{U} \alpha b$. Then, we set $\delta = 2 \exp \big( - 2 n^* \eta^2 \min( \mathcal{U}^2 \alpha^2 b^2, \Delta^2 ) \big)$ from Lemma~\ref{lemma:var_concentration}, which gives the result through simple calculations.
\end{proof}

The Proposition~\ref{prop:number_of_samples} proves that there is little incentive to choose the policy gradient estimator $\widehat{\nabla}_{\vtheta} \mathcal{E}_{\vcmp}^\alpha$ when the number of trajectories is lower than $n^*$, as its estimation error would exceed the bias introduced by the alternative estimator $\widehat{\nabla}^b_{\vtheta} \mathcal{E}_{\vcmp}^\alpha$. Unfortunately, it is not easy to compute $n^*$ in our setting, as we do not assume to know the distribution of $H_\tau$, but the requirement is arguably seldom matched in practice. 

Moreover, we can empirically show that the baseline $b = - \var (H_\tau)$ might benefit the variance of the policy gradient estimation, at the expense of the additional bias which is anyway lower than the estimation error of $\widehat{\nabla}_{\vtheta} \mathcal{E}_{\vcmp}^\alpha$.
In Figure~\ref{fig:gridslope_baseline} (left), we can see that the exploration performance $\mathcal{E}_{\vcmp}^\alpha$ obtained by $\alpha$MEPOL with and without the baseline is essentially the same in the illustrative \emph{GridWorld with Slope} domain. Whereas Figure~\ref{fig:gridslope_baseline} (right) suggests a slightly inferior variance for the policy gradient estimate employed by $\alpha$MEPOL with the baseline.
\begin{figure*}[ht]
    \centering
    \includegraphics[scale=1]{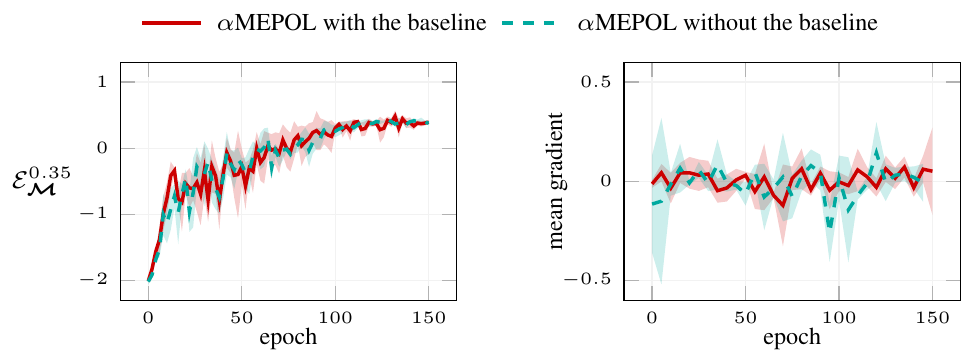}
    \caption{Comparison of the exploration performance $\mathcal{E}_{\vcmp}^{0.35}$ (left) and sampled gradients of the policy mean (right) achieved by $\alpha$MEPOL ($\alpha=0.35$) with and without the baseline $b = - \var (H_\tau)$ in the policy gradient estimation \eqref{eq:policy_gradient_estimator}. We provide 95\% c.i. over 4 runs.}
    \label{fig:gridslope_baseline}
\end{figure*}

\subsection{Importance Weighted Entropy Estimation}
As done in~\citep{mutti2020policy}, we build on the estimator in (\ref{eq:knn_entropy_estimator}) to consider the case in which the target policy $\pi_{\vtheta'}$ differs from the sampling policy $\pi_{\vtheta}$. The idea is to combine two successful policy-search methods. The first one is POIS~\citep{metelli2018policy}, to perform the optimization offline via importance sampling, allowing for an efficient exploitation of the samples collected with previous policies. We thus adopt an Importance-Weighted (IW) entropy estimator~\citep{ajgl2011differential} of the form
\begin{equation}
    \widehat{H}_{\tau_i}^{\text{IW}} = -\sum_{t=0}^{T - 1} \frac{\sum_{j \in \mathcal{N}_t^k} w_j}{k} \ln \frac{\Gamma (\frac{p}{2} + 1) \sum_{j \in \mathcal{N}_t^k} w_j }{\big\| s_{t,\tau_i} - s^{k\text{-NN}}_{t, \tau_i} \big\|^p \ \pi^{\frac{p}{2}}} + \ln k - \Psi(k),
    \label{eq:iw_entropy_estimator}
\end{equation}
where $\ln k - \Psi(k)$ is a bias correction term in which $\Psi$ is the Digamma function, $\mathcal{N}_i^k$ is the set of indices of the k-NN of $s_{t,\tau_i}$, and $w_j$ are the normalized importance weights of samples $s_{j,\tau_i}$. To compute these importance weights we consider a dataset $\mathcal{D} = \lbrace s_{t, \tau_i} \rbrace_{t=0}^{T - 1}$ by looking each state encountered in a trajectory as an unweighted particle. Then, we expand it as $\mathcal{D}_{\tau_i} = \lbrace (\tau_{i,t}, s_t) \rbrace_{t=0}^{T - 1}$, where $\tau_{i, t} = (s_{0,\tau_i}, \ldots, s_{t,\tau_i})$ is the portion of the trajectory that leads to state $s_{t, \tau_i}$. This allows to associate each particle $s_{t, \tau_i}$ to its importance weight $\widehat{w}_t$ and normalized importance weight $w_t$ for any pair of target ($\pi_{\vtheta'}$) and sampling ($\pi_{\vtheta}$) policies:
\begin{align*}
    &\widehat{w}_t = \frac{p(\tau_{i,t}|\pi_{\vtheta'})}{p(\tau_{i,t}|\pi_{\vtheta})} = \prod_{z=0}^{t} \frac{\pi_{\vtheta'}(a_{z, \tau_i}|s_{z, \tau_i})}{\pi_{\vtheta}(a_{z, \tau_i}|s_{z, \tau_i})},
    &w_t = \frac{\widehat{w}_t}{\sum_{n=0}^{T-1} \widehat{w}_n}.
\end{align*}
The estimator in (\ref{eq:iw_entropy_estimator}) is then optimized via gradient ascent within a trust-region around the current policy. The trust-region constraint is obtained by imposing
\begin{equation*}
    \widehat{D}_{KL}(\pi_{\vtheta'}||\pi_{\vtheta}) = \frac{1}{T} \sum_{t=0}^{T-1} \ln \frac{k/T}{\sum_{j \in \mathcal{N}_t^k} w_j} \le \delta,
    \label{eq:kl_estimator}
\end{equation*}
where $\widehat{D}_{KL}(\pi_{\vtheta'}||\pi_{\vtheta})$ is a non-parametric IW k-NN estimate of the Kullback-Leibler (KL) divergence~\citep{ajgl2011differential}. Its value is computed as in \citep{mutti2020policy}, by considering the entire batch of trajectories collected to execute the off-policy optimization steps as a single trajectory.

\subsection{Algorithmic Details of $\alpha$MEPOL}
In this section, we provide an extended pseudocode (Algorithm~\ref{alg:mementoapx}) of $\alpha$MEPOL, along with some additional comments.

\begin{algorithm}[tb]
\caption{$\alpha$MEPOL}
\label{alg:mementoapx}
\textbf{Input}: initial policy $\pi_{\bm{\theta}_0}$, exploration horizon $T$, number of trajectories $N$, batch-size $B$, percentile $\alpha$, learning rate $\beta$, trust-region threshold $\delta$, sampling distribution $p_{\vcmp}$\\
\textbf{Output}: exploration policy $\pi_{\vtheta_h}$
\begin{algorithmic}[1] %[1] enables line numbers
\FOR{epoch $= 0,1,\ldots$, until convergence}
\FOR{$i = 1,2,\ldots,N$}
\STATE{sample an environment $\mathcal{M}_i \sim p_{\vcmp}$}
\FOR{$j = 1,2,\ldots,B$}
\STATE{sample a trajectory $\tau_j \sim p_{\pi_{\vtheta}, \cmp_i}$ of length $T$}
\ENDFOR
\ENDFOR
\STATE{initialize dataset $\mathcal{D} = \emptyset$, off-policy step $h = 0$ and $\vtheta_h = \vtheta$}
\WHILE{$\widehat{D}_{KL}(\pi_{\vtheta_0}||\pi_{\vtheta_h}) \le \delta$}
\FOR{$j = 1,2,\ldots,B$}
\STATE{estimate $H_{\tau_j}$ with~\eqref{eq:iw_entropy_estimator}}
\STATE{append $\widehat{H}_{\tau_j}$ to $\mathcal{D}$}
\ENDFOR
\STATE{sort $\mathcal{D}$ and split it in $\mathcal{D}_{\alpha}$ and $\mathcal{D}_{1-\alpha}$}
\STATE{compute a gradient step $\vtheta_{h+1} = \vtheta_h + \beta \widehat{\nabla}_{\vtheta_h} \mathcal{E}_{\cmp}^{\alpha} (\pi_{\vtheta_h})$}
\STATE{$h \gets h+1$}
\ENDWHILE
\STATE{$\vtheta \gets \vtheta_h$}
\ENDFOR
\end{algorithmic}
\end{algorithm}

Given a probability distribution $p_{\vcmp}$, the algorithm operates by iteratively sampling an environment $\cmp_i \in \vcmp$ drawn according to $p_{\vcmp}$ and then sampling $B$ trajectories of length $T$ from it using $\pi_{\bm{\theta}}$, where $B$ is the dimension of each mini-batch. Then, the estimate of the entropy of each mini-batch $\widehat{H}_{\tau_j}$ is computed by means of the estimator in (\ref{eq:iw_entropy_estimator}) and appended to the dataset $\mathcal{D}$. Once the dataset $\mathcal{D}$ is obtained, we can straightforwardly derive a risk-sensitive policy update by just subsampling from it, so that to keep only the realizations below the $\alpha$-percentile. This can be easily done by sorting $\mathcal{D}$ in ascending order and considering only the $\alpha N$ first mini-batches. Then, we can compute the gradient as follows:
\begin{equation*}
    \widehat{\nabla}_{\vtheta} \mathcal{E}_{\vcmp}^\alpha (\pi_{\vtheta})
    = \frac{1}{\alpha N} \sum_{i = 1}^N  f_{\tau_i} \ \widehat{H}_{\tau_i} \ \ind (\widehat{H}_{\tau_i} \leq \evar (H_\tau)).
\end{equation*}
The operations carried out once all the trajectories have been sampled are executed in a fully off-policy manner, in which we repeat the same steps until the trust-region boundary is reached or until the number of off-policy iterations exceeds a specified limit. The reason why we introduce an additional parameter $B$, instead of considering one trajectory at a time, is due to the fact that a significant amount of samples (see the parameters in Table \ref{memento_param}) is needed to obtain a reliable estimate of the entropy, noting that the entropy estimator is only asymptotically unbiased.

\section{Experiments}
\label{apx:experiments}
In this section, we report an extensive description of the conducted experiments, with the corresponding hyperparameter values and some additional plots and experiments.

\subsection{Environments}
We use three different environments in our experiments. The first one is a custom implementation of a gridworld, coded from scratch. The second one is an adapted version of the rllab Ant-Maze environment~\citep{duan2016benchmarking}.

\subsubsection{GridWorld with Slope}
In \emph{GridWorld with Slope} (2D states, 2D actions), the agent can move inside a map composed of four rooms connected by four narrow hallways, by choosing at each step how much to move on the x and y axes. The side of the environment measures $2$ units and the maximum viable space of the agent at each step is $0.2$. Thus, the agent needs around $10$ steps to go from one side to the other on a straight line. When the agent collides with the external borders or with the internal walls, it is re-positioned according to a custom function. This is done not only to make the interaction more realistic, but also to limit the possibility to have a negative infinite entropy resulting from the k-NN computation, which can occur when the samples are too close and the value of the parameter $k$ is not high enough. This precaution is particularly useful in our scenario, due to the presence of a slope, and especially in the \emph{adversarial} configuration GWN, because of the initial position of the agent, which is sampled in a small square in the top-right corner. 
It is easy to see that in the first epochs in the GWN environment, the agent would repeatedly collide with the top-border, leading in general to a much more lower entropy \wrt to GWS.

The slope is applied only in the upper half of the environment, since we found this to be a good trade-off between the intention of maintaining a difference in terms of risk among the two configurations and the overall complexity of the exploration. Indeed, we noted that by applying the slope to the whole GridWorld, the advantage in terms of exploration entailed by the risk-averse approach is even higher, but it struggles to explore the bottom states of the environment with a reasonable number of samples. The slope is computed as $s \sim \mathcal{N}(\frac{\Delta_{max}}{2},\frac{\Delta_{max}}{20})$, where $\Delta_{max} = 0.2$ is the maximum step that the agent can perform.

\subsubsection{MultiGrid}
In \emph{MultiGrid}, everything works as in \emph{GridWorld with Slope}, but we indeed have 10 configurations. These environments differ for both the shape and the type of slope to which they are subject to. The \emph{adversarial} configuration is still GWN, but the slope is computed as $s \sim \mathcal{N}(\frac{\Delta_{max}}{2.6},\frac{\Delta_{max}}{20})$, where $\Delta_{max} = 0.2$. The other 9 gridworlds have instead a different arrangement of the walls (see the heatmaps in Figure~\ref{fig:multigrid_heatmaps}) and the slope, computed as $s \sim \mathcal{N}(\frac{\Delta_{max}}{3.2},\frac{\Delta_{max}}{20})$ with $\Delta_{max} = 0.2$, is applied over the entire environment. Two configurations are subject to south-facing slope, three to east-facing slope, one to south-east-facing slope and three to no slope at all.  

\subsubsection{Ant Stairs}
We adopt the Ant-Maze environment (29D states, 8D actions) of rllab~\citep{duan2016benchmarking} and we exploit its malleability to build two custom configurations which could fit our purposes. The adverse configuration consists of a narrow ascending staircase (\emph{Ant Stairs Up}) made up of an initial square (the initial position of the Ant), followed by three blocks of increasing height. The simpler configuration consists of a wide descending staircase (\emph{Ant Stairs Down}), made up of $3 \times 3$ blocks of decreasing height and a final $1 \times 3$ flat area. Each block has a side length slightly greater than the Ant size. A visual representation of such settings is provided in Figure~\ref{fig:antmaze_env_apx}. During the \emph{Unsupervised Pre-Training} phase, $\mathcal{E}^{\alpha}_{\vcmp}$ is maximized over the x,y spatial coordinates of the ant's torso.

\subsubsection{MiniGrid}
We use the MiniGrid suite~\citep{chevalier2018minimalistic}, which consists of a set of fast and light-weighted gridworld environments. The environments are partially observable, with the dimension of the agent's field of view having size $7 \times 7 \times 3$. Both the observation space $\sspace$ and the action space $\aspace$ are discrete, and in each tile of the environment there can be only one object at the same time. The set of objects is $O = \lbrace wall, floor, lava, door, key, ball, box, goal \rbrace$. The agent can move inside the grid and interact with these objects according to their properties. In particular, the actions comprise turning left, turning right, moving forward, picking up an object, dropping an object and toggling, i.e., interacting with the objects (e.g., to open a door). We exploit the suite's malleability to build two custom environments. The simpler one has a size of $18 \times 18$, and it simply contains some sparse walls. The adverse configuration is smaller, $10 \times 10$, and is characterized by the presence of a door at the top of a narrow hallway. The door is closed but not locked, meaning that the agent can open it without using a key. Moreover, we modify the movement of the agent so that the direction is given by the bottom of the triangle instead of the top. The intuition is that by doing this we are essentially changing the shape of the agent, causing an additional hurdle for the exploration.

As regards the training procedure, everything remains the same, except for two differences. The first difference is that the $k$-NN computation is performed on the representation space generated by a fixed random encoder. Note that this random encoder is not part of the policy. It is randomly initialized and not updated during the training in order to produce a more stable entropy estimate. In addition, before computing the distances, we apply to its output a random Gaussian noise $\epsilon \sim \mathcal{N}(0.001, 0.001)$ truncated in $[0, 0.001]$. We do this to avoid the aliasing problem, which occurs when we have many samples (more than $k$) in the same position, thus having zero distance and producing a negative infinite entropy estimate. The homogeneity of the MiniGrid environments in terms of features make this problem more frequent. The second difference is the addition of a bootstrapping procedure for the easy configuration, meaning that we use only a subset of the mini-batches of the easy configuration to update the policy. Especially, we randomly sample a number of mini-batches that is equal to the dimension of the $\mathcal{D}_{\alpha}$ dataset so that MEPOL uses the same number of samples of $\alpha$MEPOL. The reason why we avail this method is to avoid a clear advantage for MEPOL in learning effective representations, since it usually access more samples than $\alpha$MEPOL. Note that it is not a stretch, since we are essentially balancing the information available to the two algorithms.

\subsection{Class of Policies}
\label{apx:class_of_policies}
In all the experiments but one the policy is a Gaussian distribution with diagonal covariance matrix. It takes as input the environment state features and outputs an action vector $a \sim \mathcal{N}(\mu, \sigma^2)$. The mean $\mu$ is state-dependent and is the downstream output of a densely connected neural network. The standard deviation is state-independent and it is represented by a separated trainable vector. The dimension of $\mu$, $\sigma$, and $a$ vectors is equal to the action-space dimension of the environment. The only experiment with a different policy is the MiniGrid one, for which we adopt the architecture recently proposed by~\citep{seo2021state}. Thus, we use a random encoder made up of 3 convolutional layers with kernel 2, stride 1, and padding 0, each activated by a ReLU function, and with 16, 32 and 64 filters respectively. The first ReLU is followed by a 2D max pooling layer with kernel 2. The output of the encoder is a 64 dimensional tensor, which is then fed to a feed-forward neural network with two fully-connected layers with hidden dimension 64 and a Tanh activation function.

\subsection{Hyperparameter Values}
To choose the values of the hyper-parameters of $\alpha$MEPOL, MEPOL, and TRPO with the corresponding initialization, we mostly relied on the values reported in~\cite[][Appendix C.3, C.4]{mutti2020policy}, which have been optimized for a risk-neutral approach to the maximum state entropy setting (\ie the MEPOL baseline). By avoiding to specifically fine-tune the hyper-parameters for the $\alpha$MEPOL algorithm (and TRPO with $\alpha$MEPOL initialization), we obtain conservative comparisons between $\alpha$MEPOL and the MEPOL baseline. 

\subsubsection{Unsupervised Pre-Training}
In Table~\ref{memento_param}, we report the parameters of $\alpha$MEPOL and MEPOL that are used in the experiments described in Section~\ref{sec:experiments_learning_to_explore_illlustrative}, Section~\ref{sec:experiments_alpha_sensitivity}, Section~\ref{sec:experiments_scalability_class} and Section~\ref{sec:experiments_scalability_dimension}.

\begin{table}[H]
\caption{$\alpha$MEPOL and MEPOL Parameters for the Unsupervised Pre-Training}
\label{memento_param}
\vskip 0.15in
\begin{center}
\begin{threeparttable}
\begin{small}
\begin{sc}
\begin{tabular}{lcccr}
\toprule
& GridWorld with Slope & MultiGrid & Ant & MiniGrid \\
\midrule
Number of epochs & 150 & 50 & 400 & 300 \\
Horizon ($T$) & 400 & 400 & 400 & 150 \\
Number of traj. ($N$) & 200 & 500 & 150 & 100 \\
Mini-batch dimension ($B$) & 5 & 5 & 5 & 5 \\
$\alpha$-percentile & 0.2 & 0.1 & 0.2 & 0.2 \\
Sampling dist. ($p_{\vcmp}$) & [0.8,0.2] & [0.1,$\ldots$,0.1] & [0.8,0.2] & [0.8,0.2] \\
KL threshold ($\delta$) & 15 & 15 & 15 & 15 \\
Learning rate ($\beta$) & $10^{-5}$ & $10^{-5}$ & $10^{-5}$ & $10^{-5}$ \\
Number of neighbors ($k$) & 30 & 30 & 500 & 50 \\
Policy hidden layer sizes & (300,300) & (300,300) & (400,300) & * \\
Policy hidden layer act. funct. & ReLU & ReLU & ReLU & * \\
Number of seeds & 10 & 10 & 10 & 10 \\
\bottomrule
\end{tabular}
\end{sc}
\end{small}
\begin{tablenotes}\footnotesize
\item [*] See Section~\ref{apx:class_of_policies} for full details on the architecture.
\end{tablenotes}
\end{threeparttable}
\end{center}
\vskip -0.1in
\end{table}

\subsubsection{Supervised Fine-Tuning}
In Table~\ref{trpo_param}, we report the TRPO parameters that are used in the experiments described in Section~\ref{sec:experiments_rl_illustrative}, Section~\ref{sec:experiments_scalability_class}, Section~\ref{sec:experiments_scalability_dimension} and Section~\ref{sec:experiments_meta}.

\begin{table}[H]
\caption{TRPO Parameters for the Supervised Fine-Tuning}
\label{trpo_param}
\vskip 0.15in
\begin{center}
\begin{small}
\begin{sc}
\begin{tabular}{lcccr}
\toprule
& GridWorld with Slope & MultiGrid & Ant & MiniGrid \\
\midrule
Number of Iter. & 100 & 100 & 100 & 200 \\
Horizon & 400 & 400 & 400 & 150 \\
Sim. steps per Iter. & $1.2 \times 10^4$ & $1.2 \times 10^4$ & $4 \times 10^4$ & $7.5 \times 10^3$ \\
$\delta_{KL}$ & $10^{-4}$ & $10^{-4}$ & $10^{-2}$ & $10^{-4}$ \\
Discount ($\gamma$) & 0.99 & 0.99 & 0.99 & 0.99 \\
Number of seeds & 50 & 50 & 8 & 13 \\
Number of goals & 50 & 50 & 8 & 13 \\
\bottomrule
\end{tabular}
\end{sc}
\end{small}
\end{center}
\vskip -0.1in
\end{table}

\subsubsection{Meta-RL}
In Table~\ref{maml_param} and Table~\ref{tab:diayn_param}, we report the MAML and DIAYN parameters that are used in the experiments described in Section~\ref{sec:experiments_meta}, in order to meta-train a policy on the \emph{GridWorld with Slope} and \emph{MultiGrid} domains. For MAML, we adopted the codebase at \url{https://github.com/tristandeleu/pytorch-maml-rl}, while for DIAYN, we used the original implementation.

\begin{table}[H]
\caption{MAML Parameters for the Meta-Training}
\label{maml_param}
\vskip 0.15in
\begin{center}
\begin{small}
\begin{sc}
\begin{tabular}{lcccr}
\toprule
& GridWorld with Slope & MultiGrid \\
\midrule
Number of batches & 200 & 200 \\
Meta batch size & 20 & 20 \\
Fast batch size & 30 & 30 \\
Num. of Grad. Step & 1 & 1 \\
Horizon & 400 & 400 \\
Fast learning rate & 0.1 & 0.1 \\
Policy hidden layer sizes & (300,300) & (300,300) \\
Policy hidden layer act. function & ReLU & ReLU \\
Number of seeds & 8 & 8 \\
\bottomrule
\end{tabular}
\end{sc}
\end{small}
\end{center}
\vskip -0.1in
\end{table}

\begin{table}[H]
\caption{DIAYN Parameters}
\label{tab:diayn_param}
\vskip 0.15in
\begin{center}
\begin{small}
\begin{sc}
\begin{tabular}{lcccr}
\toprule
& GridWorld with Slope & MultiGrid \\
\midrule
Number of epochs & 1000 & 1000 \\
Horizon & 400 & 400 \\
Number of skills & 20 & 20 \\
Learning rate & $3 \times 10^{-4}$ & $3 \times 10^{-4}$ \\
Discount ($\gamma$) & 0.99 & 0.99 \\
Policy hidden layer sizes & (300,300) & (300,300) \\
Policy hidden layer act. function & ReLU & ReLU \\
Number of seeds & 8 & 8 \\
\bottomrule
\end{tabular}
\end{sc}
\end{small}
\end{center}
\vskip -0.1in
\end{table}

\subsection{Counterexamples: When Percentile Sensitivity Does Not Matter}
\label{apx:counterexample}
\begin{figure*}[ht]
    \centering
    \begin{subfigure}[t]{0.25\textwidth}
        \centering
        \includegraphics[scale=1, valign=t]{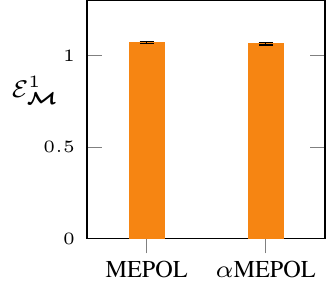}
        \caption{$p_{\vcmp} = [0.8, 0.2]$}
        \label{fig:counterexample_all}
    \end{subfigure}
    \begin{subfigure}[t]{0.25\textwidth}
        \centering
        \includegraphics[scale=1, valign=t]{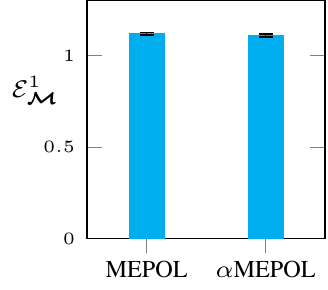}
        \caption{$p_{\vcmp} = [1, 0]$}
        \label{fig:counterexample_env_1}
    \end{subfigure}
    \begin{subfigure}[t]{0.25\textwidth}
        \centering
        \includegraphics[scale=1, valign=t]{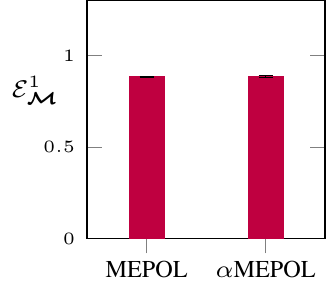}
        \caption{$p_{\vcmp} = [0, 1]$}
        \label{fig:counterexample_env_2}
    \end{subfigure}
    \caption{Comparison of the exploration performance $\mathcal{E}_{\vcmp}^1$ obtained by $\alpha$MEPOL ($\alpha = 0.35$) and MEPOL in the \emph{GridWorld Counterexample} domain. The polices are trained (50 epochs, $8 \times 10^4$ samples per epoch) on the configuration \textbf{(a)} and tested on \textbf{(a, b, c)}. We provide 95\% c.i. over 4 runs.}
    \label{fig:counterexample}
\end{figure*}
In this section, we provide a couple of convenient example to confirm the fact that there are classes of environments in which we would not need any particularly smart solution for the multiple environments problem, beyond a naïve, risk-neutral approach. First, we consider two GridWorld environments that differ for the shape of the traversable area, sampled according to $p_{\vcmp} = [0.8, 0.2]$, and we run $\alpha$MEPOL with $\alpha = 0.35$ and MEPOL, obtaining the two corresponding exploration policies. In Figure~\ref{fig:counterexample} we show the performance (measured by $\mathcal{E}_{\vcmp}^1$) obtained by executing those policies on each setting. Clearly, regardless of what configuration we consider, there is no advantage deriving from the use of a risk-averse approach as $\alpha$MEPOL, meaning that the class of environments $\vcmp$ is balanced in terms of hardness of exploration.

\begin{figure*}[ht]
    \centering
    \begin{subfigure}[t]{0.2\textwidth}
        \centering
        \includegraphics[scale=0.2, valign=b]{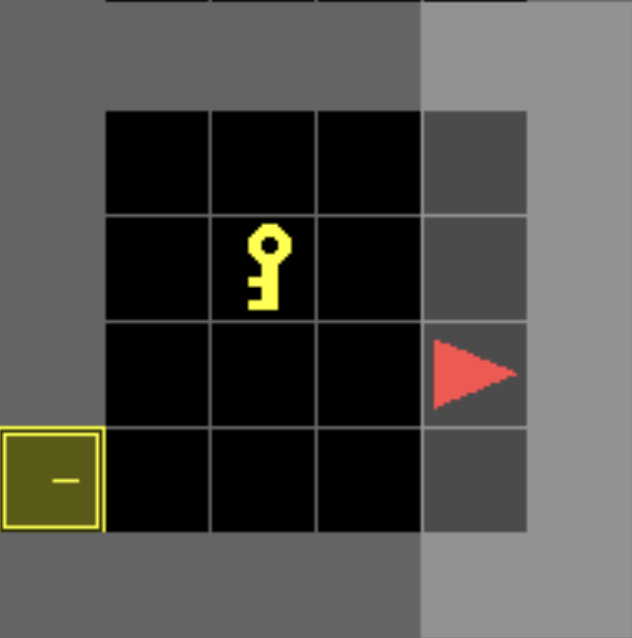}
        \caption{Unlock-v0}
        \label{fig:counterexample_unlock}
    \end{subfigure}
    \begin{subfigure}[t]{0.4\textwidth}
        \centering
        \includegraphics[scale=0.45, valign=b]{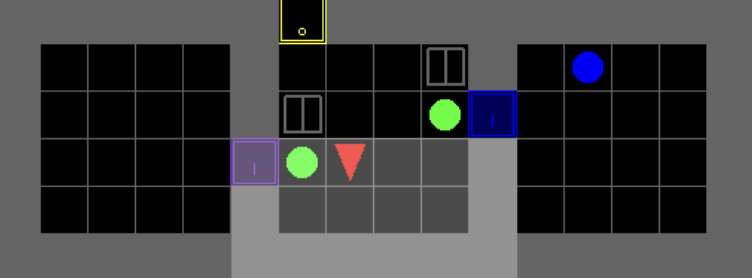}
        \caption{ObstructedMaze-2Dlhb-v0}
        \label{fig:counterexample_obstructed}
    \end{subfigure}
    \begin{subfigure}[t]{0.25\textwidth}
        \centering
        \includegraphics[scale=1, valign=b]{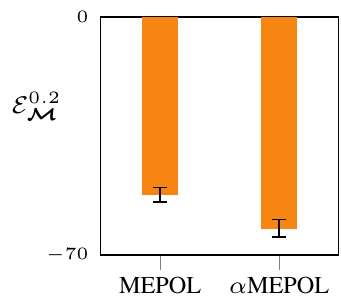}
        \caption{$p_{\vcmp} = [0.8, 0.2]$}
        \label{fig:counterexample_minigrid}
    \end{subfigure}
    \caption{Comparison of the exploration performance $\mathcal{E}_{\vcmp}^{0.2}$ obtained by $\alpha$MEPOL ($\alpha = 0.2$) and MEPOL \textbf{(c)} in a \emph{MiniGrid} domain with the configurations \emph{Unlock-v0} \textbf{(a)} and \emph{ObstructedMaze-2Dlhb-v0} \textbf{(b)}. We provide 95\% c.i. over 8 runs.}
    \label{fig:counterexample2}
\end{figure*}
In the second counterexample, we consider different configurations of the MiniGrid~\citep{chevalier2018minimalistic} domain that we considered in Section~\ref{sec:experiments_scalability_vision}. Especially, we aim to show that configurations that are visibly different from an human perspective are sometimes not really challenging from a multiple environments standpoint. Indeed, this setting would be challenging only if the policy that the agent should deploy to explore one configuration is significantly different to the one needed to explore another configuration. In this case, the agent should trade-off the performance in one configuration and the other.
As we show in Figure~\ref{fig:counterexample2}, the combination of Unlock-v0 and ObstructedMaze-2Dlhb-v0 does not have this feature, and the MEPOL baseline is able to find a policy that works well in both the configurations.

\subsection{Further Details on Meta-RL Experiments}
\label{apx:meta_experiments_details}
\begin{figure*}[ht!]
    \centering
    \begin{subfigure}[t]{0.4\textwidth}
        \centering
        \includegraphics[scale=1, valign=b]{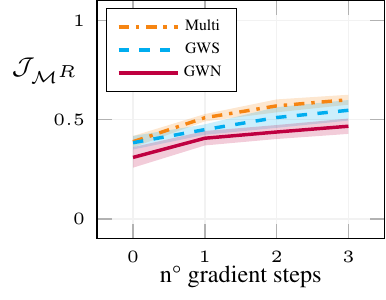}
        \caption{MAML: Fast Adaptation}
        \label{fig:adaptation}
    \end{subfigure}
    \begin{subfigure}[t]{0.4\textwidth}
        \centering
        \includegraphics[scale=1, valign=b]{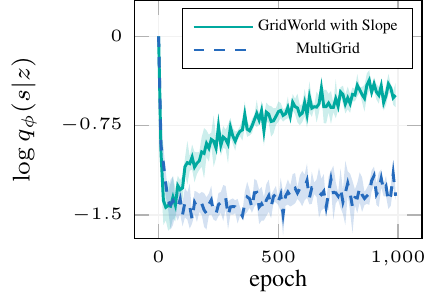}
        \caption{DIAYN: Discriminability}
        \label{fig:discriminability}
    \end{subfigure}
    \caption{We illustrate the fast-adapting behavior of MAML in the \emph{GridWorld with Slope} \textbf{(a)}, and the skills discriminability of DIAYN as a function of learning epochs \textbf{(b)}. We provide 95\% c.i. over 8 runs.}
    \label{fig:meta_details}
\end{figure*}
In this section, we provide additional details on the experiments of Section~\ref{sec:experiments_meta}. Especially, we show that MAML does perform well on its own objective, which is to learn a fast-adapting policy during meta-training (Figure~\ref{fig:adaptation}). Instead, in Figure~\ref{fig:discriminability} we highlight the performance measure of DIAYN~\cite{eysenbach2018diversity}. In particular, the more $\log q_{\phi} (s | z)$ grows with the learning epochs, the better is the intrinsic reward we feed to MAML+DIAYN. Clearly, DIAYN struggles to deal with the larger \emph{MultiGrid} class of environments, which explains the inferior performance of MAML+DIAYN in this domain.

\subsection{Additional Visualizations}
In this section, we provide some additional visualizations, which are useful to better understand some of the domains used in the experiments of Section~\ref{sec:experiments}. In Figure~\ref{fig:multigrid_heatmaps} we report the state-visitation frequencies achieved by $\alpha$MEPOL (Figure~\ref{fig:multigrid_memento_heatmaps}) and MEPOL (Figure~\ref{fig:multigrid_neutral_heatmaps}) in each configuration of the \emph{MultiGrid} domain. Clearly, $\alpha$MEPOL manages to obtain a better exploration in the adversarial configuration \wrt MEPOL, especially in the bottom part of the environment, which is indeed the most difficult part to visit. On the other environments, the performance is overall comparable. In Figure~\ref{fig:antmaze_env_apx} we show a render of the \emph{Ant Stairs} domain, illustrating both the environments used in the experiments of Section~\ref{sec:experiments_scalability_dimension}. Note that the front walls are hidden to allow for a better visualization. 

\begin{figure}[b]
    \centering
    \begin{subfigure}[b]{0.85\textwidth}
        \centering
        \includegraphics[width=\textwidth]{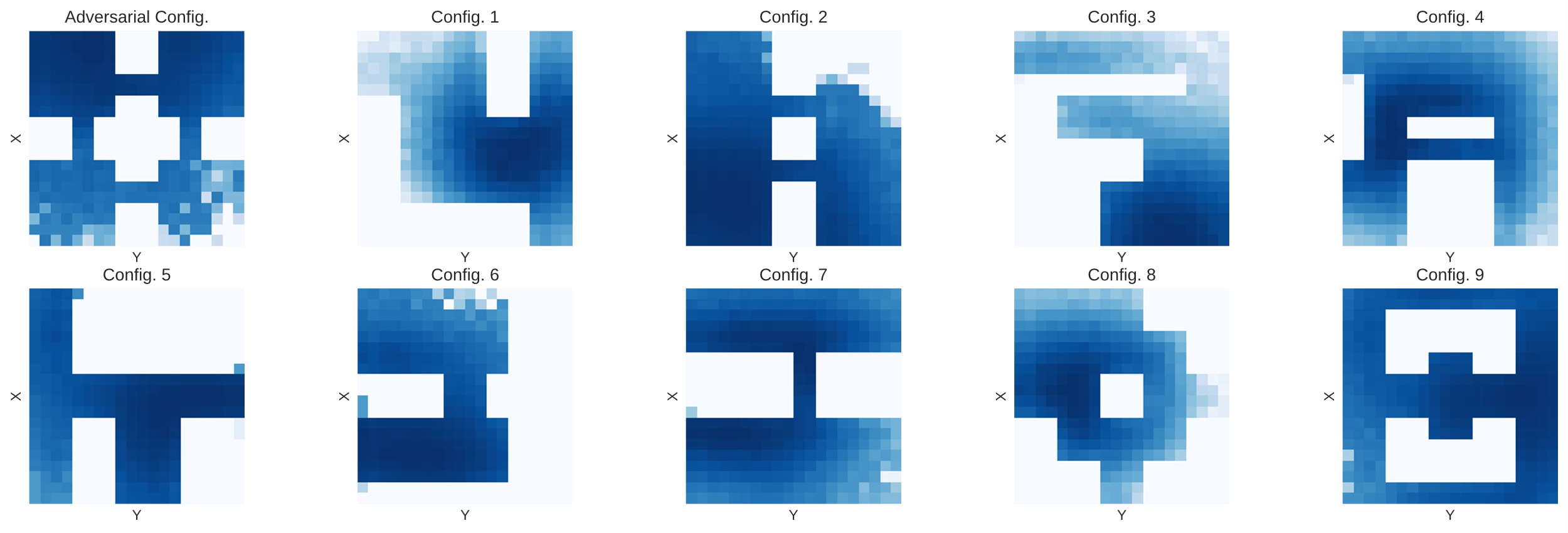}
        \caption{$\alpha$MEPOL}
        \label{fig:multigrid_memento_heatmaps}
    \end{subfigure}
    \hfill
    \begin{subfigure}[b]{0.85\textwidth}
        \centering
        \includegraphics[width=\textwidth]{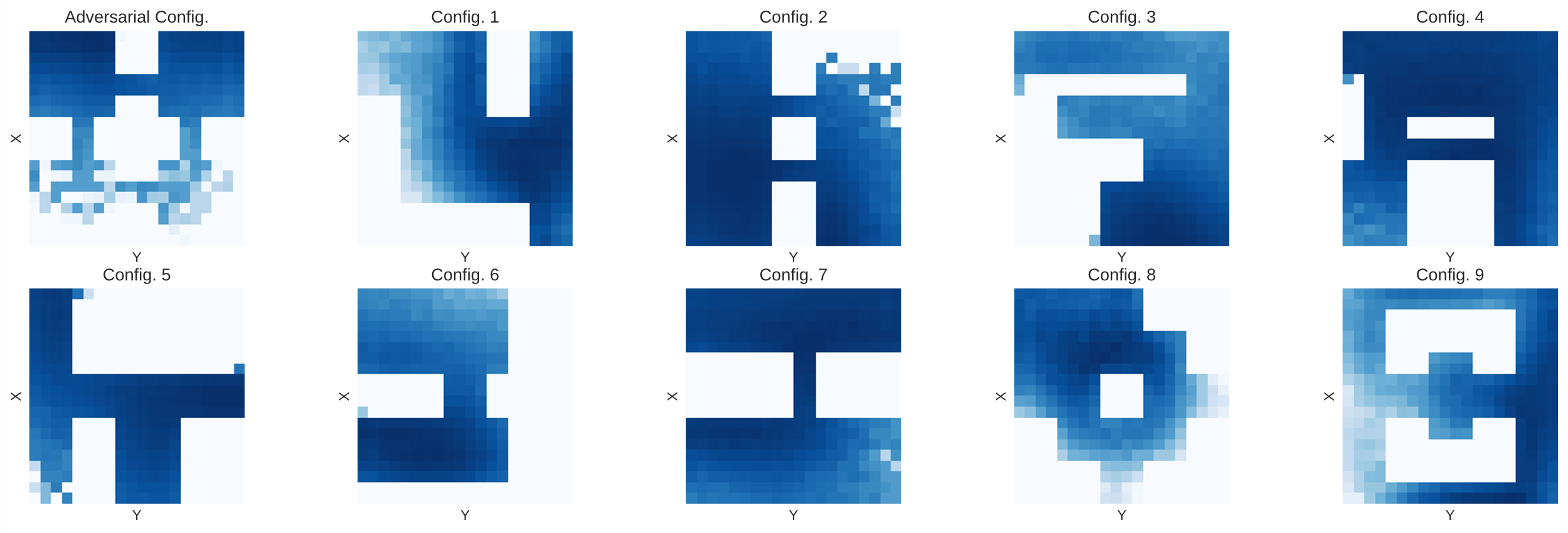}
        \caption{MEPOL}
        \label{fig:multigrid_neutral_heatmaps}
    \end{subfigure}
    \caption{Heatmaps of the state visitations (200 trajectories) induced by the exploration policies trained with $\alpha$MEPOL ($\alpha = 0.1$) \textbf{(a)} and MEPOL \textbf{(b)} in the \emph{MultiGrid} domain.}
    \label{fig:multigrid_heatmaps}
\end{figure}

\begin{figure}[t]
    \centering
    \begin{subfigure}[b]{0.3\textwidth}
        \centering
        \includegraphics[width=\textwidth]{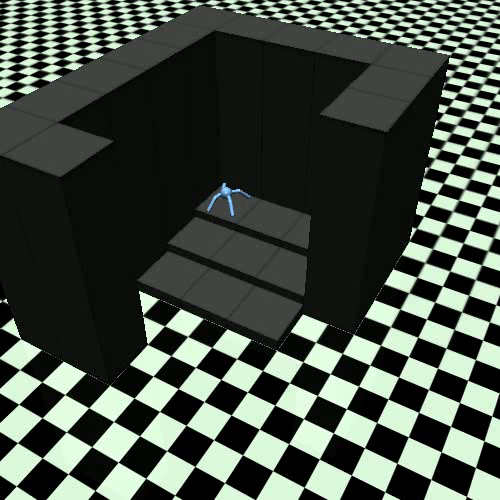}
        \caption{Ant Stairs Down}
        \label{fig:ant_stairs_down}
    \end{subfigure}
    \hspace{3em}
    \begin{subfigure}[b]{0.3\textwidth}
        \centering
        \includegraphics[width=\textwidth]{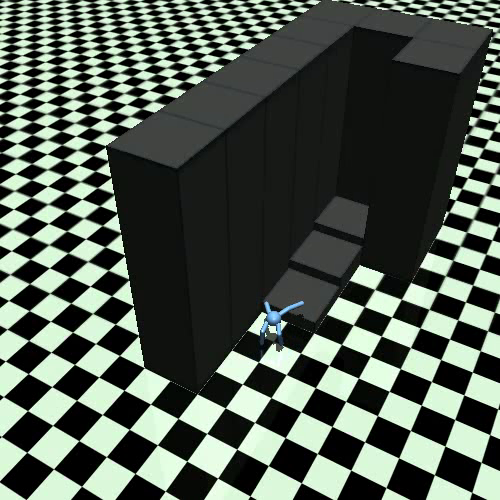}
        \caption{Ant Stairs Up}
        \label{fig:ant_stairs_up}
    \end{subfigure}
    \caption{Illustration of the \emph{Ant Stairs} domain. We show a render of the \emph{Ant Stairs Down} environment \textbf{(a)} and of the adverse \emph{Ant Stairs Up} environment \textbf{(b)}.}
    \label{fig:antmaze_env_apx}
\end{figure}

\section{Future Directions}
\label{apx:future_directions}
First, it is worth mentioning an alternative setting in which $\alpha$MEPOL can be employed with benefit (with little or no modifications).
This is the the \emph{robust unsupervised exploration} problem, in which we just have to replace the class of environments with a single CMP specified under uncertainty~\citep{satia1973markovian}.
Secondly, in this work we focused on a specific solution for an essentially multi-objective problem, by establishing a preference over the environments through the CVaR objective. Instead, a future direction could pursue learning a direct approximation of the Pareto frontier~\citep{parisi2016multi} of the exploration strategies over multiple environments. 
Another promising direction is to assume some control over the class distribution during the unsupervised pre-training process, either by an external supervisor or by the agent itself~\citep{metelli2018configurable}. 
Lastly, future work may establish regret guarantees for the reward-free exploration problem over multiple environments, in a similar flavor to the reward-free RL problem in a single environment~\citep{jin2020rewardfree}.

%%%%%%%%%%%%%%%%%%%%%%%%%%%%%%%%%%%%%%%%%%%%%%%%%%%%%%%%%%%%%%%%%%%%%%%%%%%%%%%
%%%%%%%%%%%%%%%%%%%%%%%%%%%%%%%%%%%%%%%%%%%%%%%%%%%%%%%%%%%%%%%%%%%%%%%%%%%%%%%

\end{document}